\documentclass{article}

\usepackage{xr}

\usepackage[english]{babel}
\usepackage{graphicx}

\usepackage{epsfig}
\usepackage{enumitem}

\usepackage{amssymb,amsmath}
\usepackage{amsthm}

\newtheorem{definition}{Definition}
\newtheorem{theorem}{Theorem}[section]
\newtheorem{lemma}{Lemma}
\newtheorem{cor}[theorem]{Corollary}
\usepackage{natbib}
\usepackage{mdwlist}
\newcommand{\EE}{\mathbb{E}}
\newcommand{\RR}{\mathbb{R}}
\newcommand{\NN}{\mathcal{N}}
\newcommand{\PP}{\mathbb{P}}
\newcommand{\FF}{\mathcal{F}}

\newcommand{\XX}{\mathcal{X}}
\newcommand{\YY}{\mathcal{Y}}
\newcommand{\ZZ}{\mathcal{Z}}

\newcommand{\Scal}{\mathcal{S}}
\newcommand{\sign}{\text{sign}}
\newcommand{\inner}[1]{\langle #1\rangle}
\newcommand{\norm}[1]{\left\lVert#1\right\rVert}
\newcommand{\maxnorm}[1]{\left\lVert#1\right\rVert_{\max}}

\newcommand{\sm}{\text{sim}}

\newcommand{\removed}[1]{}
\newcommand{\ewp}{\odot}
\usepackage{times}
\usepackage{graphicx} 
\usepackage{subfigure} 

\usepackage{natbib}

\usepackage{algorithm}
\usepackage{algorithmic}

\usepackage{hyperref}

\usepackage[accepted]{icml2015}

\icmltitlerunning{On Symmetric and Asymmetric LSHs for Inner Product Search}

\begin{document} 

\twocolumn[
\icmltitle{On Symmetric and Asymmetric LSHs for Inner Product Search}

\icmlauthor{Behnam Neyshabur}{bneyshabur@ttic.edu}
\icmlauthor{Nathan Srebro}{nati@ttic.edu}
\icmladdress{Toyota Technological Institute at Chicago, Chicago, IL 60637, USA}

\icmlkeywords{Locality Sensitive Hashing, MIPS, Asymmetric LSH}

\vskip 0.3in]

\begin{abstract}
  We consider the problem of designing locality sensitive hashes (LSH)
  for inner product similarity, and of the power of asymmetric hashes
  in this context.  \citet{shrivastava14} argue that there is no
  {\em symmetric} LSH for the problem and propose an {\em asymmetric}
  LSH based on different mappings for query and database points.
  However, we show there {\em does} exist a simple {\em symmetric} LSH
  that enjoys stronger guarantees and better empirical performance
  than the asymmetric LSH they suggest.  We also show a variant of the
  settings where asymmetry is in-fact needed, but there a different
  asymmetric LSH is required.
\end{abstract}

\section{Introduction}
Following \citet{shrivastava14}, we consider the problem of Maximum
Inner Product Search (MIPS): given a collection of ``database'' vectors
$\Scal\subset\RR^d$ and a query $q\in\RR^d$, find a data vector
maximizing the inner product with the query:
\begin{equation}\label{eq:MIPS}
p = \arg\max_{x\in \Scal}\;\; q^\top x
\end{equation}
MIPS problems of the form \eqref{eq:MIPS} arise, e.g.~when using
matrix-factorization based recommendation systems \citep{yehuda09,
  srebro05b, cremonesi10}, in multi-class prediction
\citep{dean13,jain09} and structural SVM \citep{joachims06,joachims09}
problems and in vision problems when scoring filters based on their
activations \citep{dean13} \citep[see][for more about
MIPS]{shrivastava14}.  In order to efficiently find approximate MIPS
solutions, \citet{shrivastava14} suggest constructing a Locality
Sensitive Hash (LSH) for inner product ``similarity''.

Locality Sensitive Hashing \citep{indyk98} is a popular tool for
approximate nearest neighbor search and is also widely used in other
settings \citep{gionis99,datar04, charikar02}.  An LSH is a random
mapping $h(\cdot)$ from objects to a small, possibly binary, alphabet,
where collision probabilities $\PP[h(x)=h(y)]$ relate to the desired
notion of similarity $\sm(x,y)$.  An LSH can in turn be used to
generate short hash words such that hamming distances between hash
words correspond to similarity between objects.  Recent studies have
also explored the power of asymmetry in LSH and binary hashing, where
two different mappings $f(\cdot),g(\cdot)$ are used to approximate
similarity, $\sm(x,y)\approx \PP[h(x)=g(y)]$ \citep{neyshabur13,
  neyshabur14}.  \citeauthor{neyshabur14} showed that even when the
similarity $\sm(x,y)$ is entirely symmetric, asymmetry in the hash may
enable obtaining an LSH when a symmetric LSH is not possible, or
enable obtaining a much better LSH yielding shorter and more accurate
hashes. 

Several tree-based methods have also been proposed for inner product search~\citep{ram12,koenigstein12,curtin13}.
\citet{shrivastava14} argue that tree-based methods, such as cone trees, are impractical in
high dimensions while the performance of LSH-based methods is in a way independent of dimension of the data. Although the exact regimes
under which LSH-based methods are superior to tree-based methods and vice versa are not fully established yet, the goal of this paper is to analyze different LSH methods and compare them with each other, rather than comparing to tree-based methods, so as to understand which LSH to use and why, in those regimes where tree-based methods are not practical.

Considering MIPS, \citet{shrivastava14} argue that there is no
symmetric LSH for inner product similarity, and propose two distinct
mappings, one of database objects and the other for queries, which
yields an asymmetric LSH for MIPS.  But the caveat is that they
consider different spaces in their positive and negative results: they
show nonexistence of a symmetric LSH over the entire space $\RR^d$,
but their asymmetric LSH is only valid when queries are normalized and
data vectors are bounded.  Thus, they do {\em not} actually show a
situation where an asymmetric hash succeeds where a symmetric hash is
not possible.  In fact, in Section \ref{sec:mips} we show a simple
{\em symmetric} LSH that is also valid under the same assumptions, and
it
even enjoys improved theoretical guarantees and empirical
performance!  This suggests that asymmetry might
actually not be required nor helpful for MIPS.

Motivated by understanding the power of asymmetry, and using this
understanding to obtain the simplest and best possible LSH for MIPS,
we conduct a more careful study of LSH for inner product similarity.
A crucial issue here is what is the space of vectors over which we
would like our LSH to be valid.  First, we show that over the entire
space $\RR^d$, not only is there no symmetric LSH, but there is also
no asymmetric LSH either (Section \ref{sec:rd}).  Second, as mentioned
above, when queries are normalized and data is bounded, a symmetric
LSH is possible and there is no need for asymmetry.  But when queries
and data vectors are bounded and queries are not normalized, we do
observe the power of asymmetry: here, a symmetric LSH is not possible,
but an asymmetric LSH exists (Section \ref{sec:sphere}).

As mentioned above, our study also yields an LSH for MIPS, which we
refer to as \textsc{simple-lsh}, which is not only symmetric but also
parameter-free and enjoys significantly better theoretical and
empirical compared to \textsc{l2-alsh(sl)} proposed by
\citet{shrivastava14}.  In Appendix~\ref{supp:sign} we show that all
of our theoretical observations about \textsc{l2-alsh(sl)} apply also
to the alternative hash \textsc{sign-lsh(sl)} put forth by
\citet{shrivastava14b}.

The transformation at the root of
\textsc{simple-lsh}  was also recently proposed by \citet{bachrah14},
who used it in a PCA-Tree data structure for speeding up the Xbox recommender system.
Here, we study the transformation as part of an LSH scheme,
investigate its theoretical properties, and compare
it to \textsc{ls-alsh(sl)}.

\section{Locality Sensitive Hashing}\label{sec:lsh}

A {\em \bf hash} of a set $\ZZ$ of objects is a random mapping from $\ZZ$
to some alphabet $\Gamma$, i.e.~a distribution over functions
$h:\ZZ\rightarrow\Gamma$.  The hash is sometimes thought of as a
``family'' of functions, where the  distribution over the
family is implicit.

When studying hashes, we usually study the behavior when comparing any
two points $x,y\in\ZZ$.  However, for our study here, it will be
important for us to make different assumptions about $x$ and
$y$---e.g., we will want to assume w.l.o.g.~that queries are
normalized but will not be able to make the same assumptions on
database vectors.  To this end, we define what it means for a hash to
be an LSH over a pair of constrained subspaces $\XX,\YY\subseteq\ZZ$.
Given a similarity function $\sm:\ZZ\times\ZZ\rightarrow\RR$, such as
inner product similarity $\sm(x,y)=x^\top y$, an LSH is defined as
follows\footnote{This is a formalization of the definition given by
  \citet{shrivastava14}, which in turn is a modification of the
  definition of LSH for distance functions \citep{indyk98}, where we
  also allow different constraints on $x$ and $y$.  Even
  though inner product similarity could be negative, this definition
  is only concerned with the positive values.}:
\begin{definition}[Locality Sensitive Hashing (LSH)]\label{def:lsh2}
  A hash is said to be a {\bf $(S,cS,p_1,p_2)$-LSH} for a
  similarity function $\sm$ {\bf over the pair of spaces
    $\XX,\YY\subseteq\ZZ$} if for any $x\in \XX$ and $y\in \YY$:
\begin{itemize}[noitemsep,topsep=0pt,parsep=0pt]
\item if $\sm(x,y) \geq S$ then $\PP_{h}[h(x)=h(y)] \geq p_1$,
\item if $\sm(x,y) \leq cS$ then $\PP_{h}[h(x)=h(y)] \leq p_2$.
\end{itemize}
When $\XX=\YY$, we say simply ``{\bf over the space $\XX$}''.
\end{definition}
Here $S>0$ is a threshold of interest, and for efficient approximate
nearest neighbor search, we need $p_1 > p_2$ and $c <1$.  In
particular, given an $(S,cS,p1,p2)$-LSH, a data structure for finding
$S$-similar objects for query points when $cS$-similar objects exist in
the database can be constructed in time $O(n^{\rho} \log n)$ and space
$O(n^{1+\rho})$ where $\rho = \frac{\log p_1}{\log p_2}$.  This
quantity $\rho$ is therefore of particular interest, as we are
interested in an LSH with minimum possible $\rho$, and we refer to it
as the {\em hashing quality}.

In Definition \ref{def:lsh2}, the hash itself is still symmetric,
i.e.~the same function $h$ is applied to both $x$ and $y$.  The only
asymmetry allowed is in the problem definition, as we allow requiring
the property for differently constrained $x$ and $y$.  This should be
contrasted with a truly asymmetric hash, where two different functions
are used, one for each space.  Formally, an {\em \bf asymmetric hash}
for a pair of spaces $\XX$ and $\YY$ is a joint distribution over
pairs of mappings $(f,g)$, $f:\XX\rightarrow\Gamma$,
$g:\YY\rightarrow\Gamma$.  The asymmetric hashes we consider will be
specified by a pair of deterministic mappings $P:\XX\rightarrow\ZZ$
and $Q:\YY\rightarrow\ZZ$ and a single random mapping
(i.e.~distribution over functions) $h:\ZZ\rightarrow\Gamma$, where
$f(x)=h(P(x))$ and $g(y)=h(Q(y))$.  Given a similarity function
$\sm:\XX\times\YY\rightarrow\RR$ we define\removed{\footnote{This is a
    formalization of the definition given by
    \citet{shrivastava14}, where we have made the distinction
    between the spaces $\XX$ and $\YY$, and the quantifiers on $x,y$,
    explicit.  We also distinguish between an asymmetric hash and an
    asymmetric notion of an LSH.}}:
\begin{definition}[Asymmetric Locality Sensitive Hashing (ALSH)]\label{def:alsh}
  An asymmetric hash is said to be an {\bf $(S,cS,p_1,p_2)$-ALSH} for a
  similarity function $\sm$ {\bf over $\XX,\YY$}
  if for any $x\in \XX$ and $y\in \YY$:
\begin{itemize}[noitemsep,topsep=0pt,parsep=0pt]
\item if $\sm(x,y) \geq S$ then $\PP_{(f,g)}[f(x)=g(y)] \geq p_1$,
\item if $\sm(x,y) \leq cS$ then $\PP_{(f,g)}[f(x)=g(y)] \leq p_2$.
\end{itemize}
\end{definition}
Referring to either of the above definitions, we also say that a hash
is an {\bf $(S,cS)$-LSH} (or {\bf ALSH}) if there exists $p_2>p_1$ such
that it is an $(S,cS,p_1,p_2)$-LSH (or ALSH).  And we say it is a {\bf
  universal LSH} (or {\bf ALSH}) if for every $S>0,0<c<1$ it is an
$(S,cS)$-LSH (or ALSH).

\section{No ALSH over $\RR^d$}\label{sec:rd}
Considering the problem of finding an LSH for inner product
similarity, \citet{shrivastava14} first observe that for any
$S>0,0<c<1$, there is no {\em symmetric} $(S,cS)$-LSH for $\sm(x,y)=x^\top
y$ over the entire space $\XX=\RR^d$, which prompted them to consider
asymmetric hashes.  In fact, we show that asymmetry doesn't help here,
as there also isn't any ALSH over the entire space:
\begin{theorem}
For any $d\geq 2$, $S>0$ and $0<c<1$ there is no asymmetric hash that is an $(S,cS)$-ALSH for inner product similarity over $\XX=\YY=\RR^d$.
\end{theorem}
\begin{proof}
  Assume for contradiction there exists some $S>0,0<c<1$ and $p_1>p_2$
  for which there exists an $(S,cS,p_1,p_2)$-ALSH $(f,g)$ for inner
  product similarity over $\RR^2$ (an ALSH for inner products over
  $\RR^d$, $d>2$, is also an ALSH for inner products over a two-dimensional
  subspace, i.e.~over $\RR^2$, and so it is enough to consider
  $\RR^2$).  Consider the following
  two sequences of points:
\begin{align*}
x_i &= [-i,1]\\
y_j &= [S(1-c),S(1-c)j+S].
\end{align*}
For any $N$ (to be set later), define the $N\times N$ matrix $Z$ as follows:
\begin{equation}
Z(i,j)=
\begin{cases}
1 & x_i^\top y_j \geq S\\
-1 & x_i^\top y_j \leq cS\\
0 & \text{otherwise}.
\end{cases}
\end{equation}
Because of the choice of $x_i$ and $y_j$, the
matrix $Z$ does not actually contain zeros, and is in-fact triangular
with $+1$ on and above the diagonal and $-1$ below it.  Consider also
the matrix $P\in\RR^{N\times N}$ of collision probabilities
$P(i,j)=\PP_{(f,g)}[f(x_i)=g(x_j)]$.  Setting $\theta=(p_1+p_2)/2<1$ and
$\epsilon=(p_1-p_2)/2>0$, the ALSH property implies that for every $i,j$:
\begin{equation}\label{eq:mc}
Z(i,j) (P(i,j)-\theta)\geq \epsilon
\end{equation}
or equivalently:
\begin{equation}\label{eq:mc}
Z \ewp \frac{P-\theta}{\epsilon}\geq 1
\end{equation}
where $\ewp$ denotes element-wise (Hadamard) product. Now, for a sign
matrix $Z$, the {\em margin complexity} of $Z$ is defined as
$mc(Z)=\inf_{Z\ewp X\geq 1}\maxnorm{X}$ \citep[see][and also for the
definition of the max-norm $\maxnorm{X}$]{srebro05}, and
we know that the margin complexity of an $N \times N$ triangular
matrix is bounded by $mc(Z)=\Omega(\log N)$ \citep{forster03},
implying
\begin{equation}
  \label{eq:maxnormPthetaeps}
 \maxnorm{(P-\theta)/\epsilon}=\Omega(\log N). 
\end{equation}
Furthermore, any collision probability matrix has max-norm
$\maxnorm{P} \leq1$ \citep{neyshabur14}, and shifting the matrix by
$0<\theta<1$ changes the max-norm by at most $\theta$, implying
$\maxnorm{P-\theta}\leq 2$, which combined with
\eqref{eq:maxnormPthetaeps} implies $\epsilon=O(1/\log N)$. For any
$\epsilon=p_1-p_2>0$, selecting a large enough $N$ we get a
contradiction.
\end{proof}
For completeness, we also include in Appendix~\ref{supp:max} a full
definition of the max-norm and margin complexity, as well as the
bounds on the max-norm and margin complexity used in the proof above.

\section{Maximum Inner Product Search}\label{sec:mips}

We saw that no LSH, nor ALSH, is possible for inner product similarity
over the entire space $\RR^d$.  Fortunately, this is not required for
MIPS.  As pointed out by \citet{shrivastava14}, we can assume the following
without loss of generality:
\begin{itemize*}
\item The query $q$ is normalized: Since given a vector $q$, the norm
  $\norm{q}$ does not affect the argmax in \eqref{eq:MIPS}, we can
  assume $\norm{q}=1$ always.
\item The database vectors are bounded inside the unit
  sphere: We assume $\norm{x} \leq 1$ for all $x\in\Scal$.
  Otherwise we can rescale all vectors without changing the argmax.
\end{itemize*}
We cannot, of course, assume the vectors $x$ are normalized.  This
means we can limit our attention to the behavior of the hash over
$\XX_\bullet=\left\{x\in\RR^d \;\middle|\;\norm{x}\leq 1\right\}$ and
$\YY_\circ=\left\{q\in\RR^d \;\middle|\;\norm{q}=1\right\}$. Indeed,
\citet{shrivastava14} establish the existence of an {\em
  asymmetric} LSH, which we refer to as \textsc{l2-alsh(sl)}, over
this pair of database and query spaces.  Our main result in this
section is to show that in fact there does exists a simple,
parameter-free, universal, {\em symmetric} LSH, which we refer to as
\textsc{simple-lsh}, over $\XX_\bullet,\YY_\circ$.  We see then that
we {\em do} need to consider the hashing property asymmetrically (with
different assumptions for queries and database vectors), but the same
hash function can be used for both the database and the queries and
there is no need for two different hash functions or two different
mappings $P(\cdot)$ and $Q(\cdot)$.

But first, we review \textsc{l2-alsh(sl)} and note that it is not
universal---it depends on three parameters and no setting of the
parameters works for all thresholds $S$.  We also compare our
\textsc{simple-lsh} to \textsc{l2-alsh(sl)} (and to the 
recently suggested \textsc{sign-alsh(sl)}) both in terms of the hashing quality
$\rho$ and empirically of movie recommendation data sets.
\subsection{L2-ALSH(SL)}

For an integer parameter $m$, and real valued parameters $0<U<1$ and
$r>0$, consider the following pair of mappings: 
\begin{equation}
  \label{eq:theirPQ}
  \begin{aligned}
    P(x)&=[Ux;\norm{Ux}^2;\norm{Ux}^4;\dots;\norm{Ux}^{2^m}]\\
    Q(y)&=[y;1/2;1/2;\dots;1/2],
  \end{aligned}
\end{equation}
combined with the standard $L_2$ hash function
\begin{equation}
  \label{eq:l2h}
  h_{a,b}^{L_2}(x) = \bigg\lfloor \frac{a^\top x + b}{r} \bigg\rfloor
\end{equation}
where $a\sim\NN(0,I)$ is a spherical multi-Gaussian random vector,
$b\sim \mathcal{U}(0,r)$ is a uniformly distributed random variable on $[0,r]$.\removed{We know that for any $x,y\in \RR$, the collision probability of the hash $h_{a,b}^{L_2}$ can be written as \citep{datar04}:
\begin{eqnarray}\label{Fr}
\nonumber
\PP\big[h_{a,b}^{L_2}(x)=h_{a,b}^{L_2}(y)\big] &=&1-2\Phi(-r/\delta)-\frac{1-e^{-\frac{1}{2}(\frac{r}{\delta})^2}}{\sqrt{2\pi} (r/\delta)}\\
&=& \FF_r(\delta)
\end{eqnarray}
where $\delta=\norm{x-y}$ and $\Phi(x)=\int_{-\infty}^x
\frac{1}{\sqrt{2\pi}}e^{-x^2/2}dx$ is the cumulative density function
of standard normal distribution.}  The alphabet $\Gamma$ used is the
integers, the intermediate space is $\ZZ=\RR^{d+m}$ and the asymmetric
hash \textsc{l2-alsh(sl)}, parameterized by
$m,U$ and $r$, is then given by
\begin{equation}
  \label{eq:MIPS-ALSH}
(f(x),g(q))=(h_{a,b}^{L_2}(P(x)),h_{a,b}^{L_2}(Q(q))).
\end{equation}
\citet{shrivastava14}
establish\footnote{\citet{shrivastava14} have the scaling by $U$
  as a separate step, and state their hash as an $(S_0,cS_0)$-ALSH
  over $\{\norm{x}\leq U\},\{\norm{q}=1\}$, where the threshold
  $S_0=US$ is also scaled by $U$.  This is equivalent to the
  presentation here which integrates the pre-scaling step, which also
  scales the threshold, into the hash.} that for any $0<c<1$ and
$0<S<1$, there exists $0<U<1$, $r>0$, $m\geq 1$, such that
\textsc{l2-alsh(sl)} is an $(S,cS)$-ALSH over $\XX_\bullet,\YY_\circ$.  They
furthermore calculate the hashing quality $\rho$ as a function of
$m,U$ and $r$, and numerically find the optimal $\rho$ over a grid of
possible values for $m,U$ and $r$, for each choice of $S,c$.

Before moving on to presenting a symmetric hash for the problem, we
note that \textsc{l2-alsh(sl)} is not {\em universal} (as defined at
the end of Section \ref{sec:lsh}).  That is, not only might the
optimal $m,U$ and $r$ depend on $S,c$, but in fact there is no choice
of the parameters $m$ and $U$ that yields an ALSH for all $S,c$, or even for
all ratios $c$ for some specific threshold $S$ or for all thresholds
$S$ for some specific ratio $c$. This is unfortunate, since in MIPS
problems, the relevant threshold $S$ is the maximal inner product
$\max_{x\in \Scal} q^\top x$ (or the threshold inner product if we are
interested in the ``top-$k$'' hits), which typically varies with the
query.  It is therefore desirable to have a single hash that works for
all thresholds.
\begin{lemma}\label{lem:l2alshbound}
  For any $m,U,r$, and for any $0<S<1$ and
  $$1-\frac{U^{2^{m+1}-1}(1-S^{2^{m+1}})}{2S}\leq c < 1,$$
  \textsc{l2-alsh(sl)} is not an $(S,cS)$-ALSH for inner product
  similarity over $\XX_\bullet=\left\{ x \middle| \norm{x}\leq 1 \right\}$ and
  $\YY_\circ=\left\{ q \middle| \norm{q}=1 \right\}$.
\end{lemma}
\begin{proof}
  Assume for contradiction that it is an $(S,cS)$-ALSH.  For any query
  point $q\in \YY_\circ$, let $x\in \XX_\bullet$ be a vector s.t.~$q^\top x=S$ and
  $\|x\|_2=1$ and let $y=cSq$, so that $q^\top y = cS$. We have that:
\begin{eqnarray*}
p_1\leq\PP\big[h_{a,b}^{L_2}(P(x))=h_{a,b}^{L_2}(Q(q))\big] =\FF_r(\|P(x)-Q(q)\|_2)\\
p_2\geq\PP\big[h_{a,b}^{L_2}(P(y))=h_{a,b}^{L_2}(Q(q))\big] =\FF_r(\|P(y)-Q(q)\|_2)
\end{eqnarray*}
where $\FF_r(\delta)$ is a monotonically decreasing function of
$\delta$ \citep{datar04}.  To get a contradiction it is therefor
enough to show that  $\norm{P(y)-Q(q)}^2\leq \norm{P(x)-Q(q)}^2$. We have:
\begin{align*}
\norm{P(y)-Q(q)}^2 &= 1+\frac{m}{4}+ \norm{y}^{2^{m+1}}-2q^\top y\\
&= 1+\frac{m}{4}+ (cSU)^{2^{m+1}}-2cSU\\
\intertext{using $1-\frac{U^{2^{m+1}-1}(1-S^{2^{m+1}})}{2S}\leq c < 1$:}
&< 1+\frac{m}{4}+ (SU)^{2^{m+1}}-2cSU\\
&\leq 1+\frac{m}{4}+U^{2^{m+1}} -2SU\\
&=\norm{P(x)-Q(q)}^2\qedhere
\end{align*}
\end{proof}
\begin{cor}
  For any $U,m$ and $r$, \textsc{l2-alsh(sl)} is not a universal ALSH
  for inner product similarity over $\XX_\bullet=\left\{ x \middle|
    \norm{x}\leq 1 \right\}$ and $\YY_\circ=\left\{ q \middle| \norm{q}=1
  \right\}$.  Furthermore, for any $c<1$, and any choice of $U,m,r$
  there exists $0<S<1$ for which \textsc{l2-alsh(sl)} is not an
  $(S,cS)$-ALSH over $\XX_\bullet,\YY_\circ$, and for any $S<1$ and any choice of
  $U,m,r$ there exists $0<c<1$ for which \textsc{l2-alsh(sl)} is not
  an $(S,cS)$-ALSH over $\XX_\bullet,\YY_\circ$.
\end{cor}

In Appendix \ref{supp:sign}, we show a similar non-universality
result also for \textsc{sign-alsh(sl)}.
\subsection{SIMPLE-LSH}
We propose here a simpler, parameter-free, symmetric LSH, which we
call \textsc{simple-lsh}.

For $x\in\RR^d$, $\norm{x}\leq 1$, define
$P(x)\in \RR^{d+1}$ as follows~\cite{bachrah14}:
\begin{equation}
  \label{eq:simpleP}
  P(x) = \big[x;\sqrt{1-\|x\|_2^2}\big]
\end{equation}
For any $x\in\XX_\bullet$ we have $\norm{P(x)}=1$, and for any
$q\in\YY_\circ$, since $\norm{q}=1$, we have:
\begin{equation}\label{eq:PqPx}
P(q)^\top P(x) = \big[q; 0\big]^\top \big[x;\sqrt{1-\|x\|_2^2}\big] = q^\top x
\end{equation}
Now, to define the hash \textsc{simple-lsh}, take a spherical random
vector $a\sim\NN(0,I)$ and consider the following random mapping into the binary
alphabet $\Gamma=\{\pm 1\}$:
\begin{equation}\label{eq:ha}
h_{a}(x) = \sign(a^\top P(x)).
\end{equation}

\begin{figure*}[t!]
\hbox{ \centering
\setlength{\epsfxsize}{0.332\textwidth}
\epsfbox{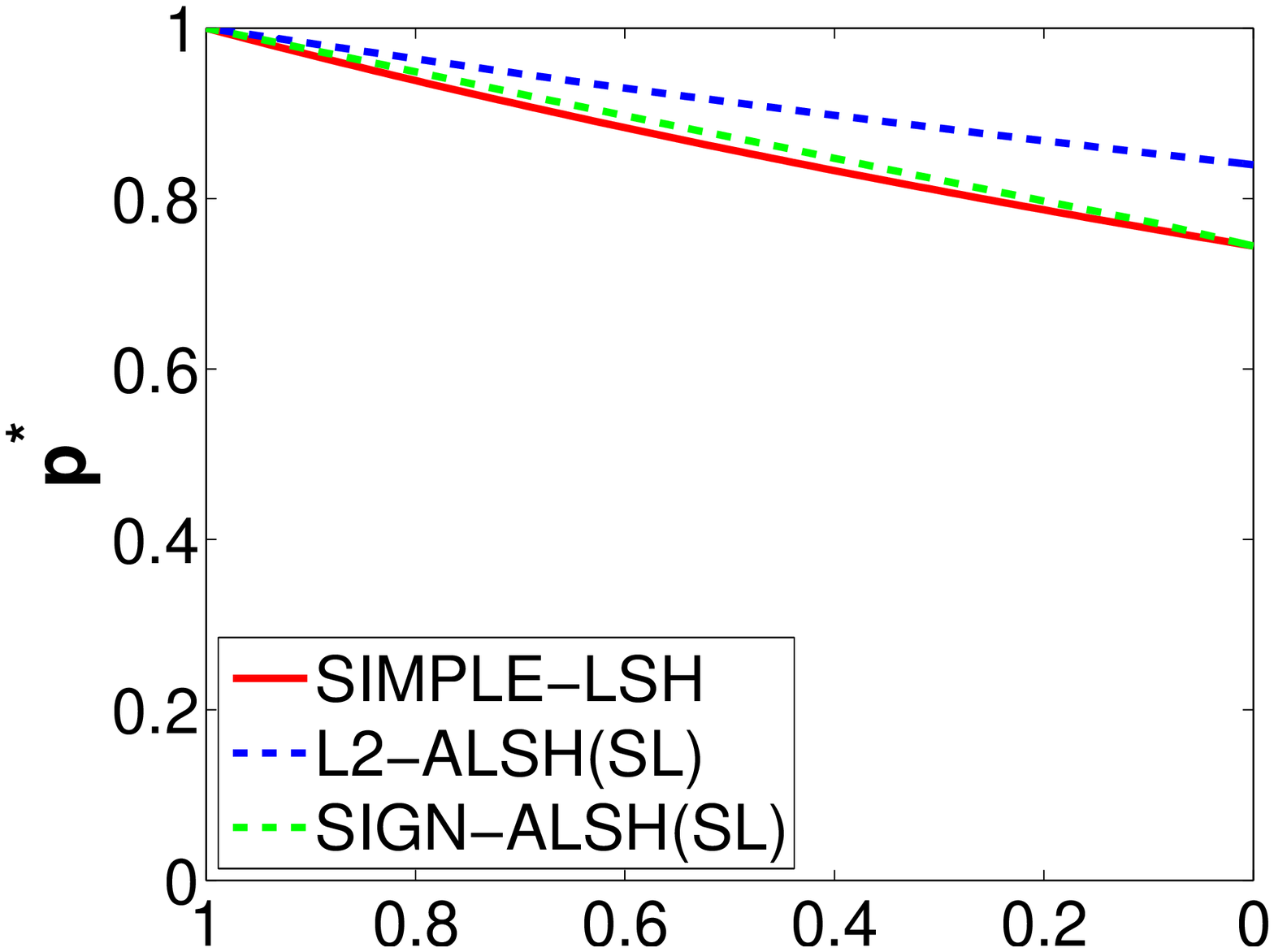}
\setlength{\epsfxsize}{0.31\textwidth}
\epsfbox{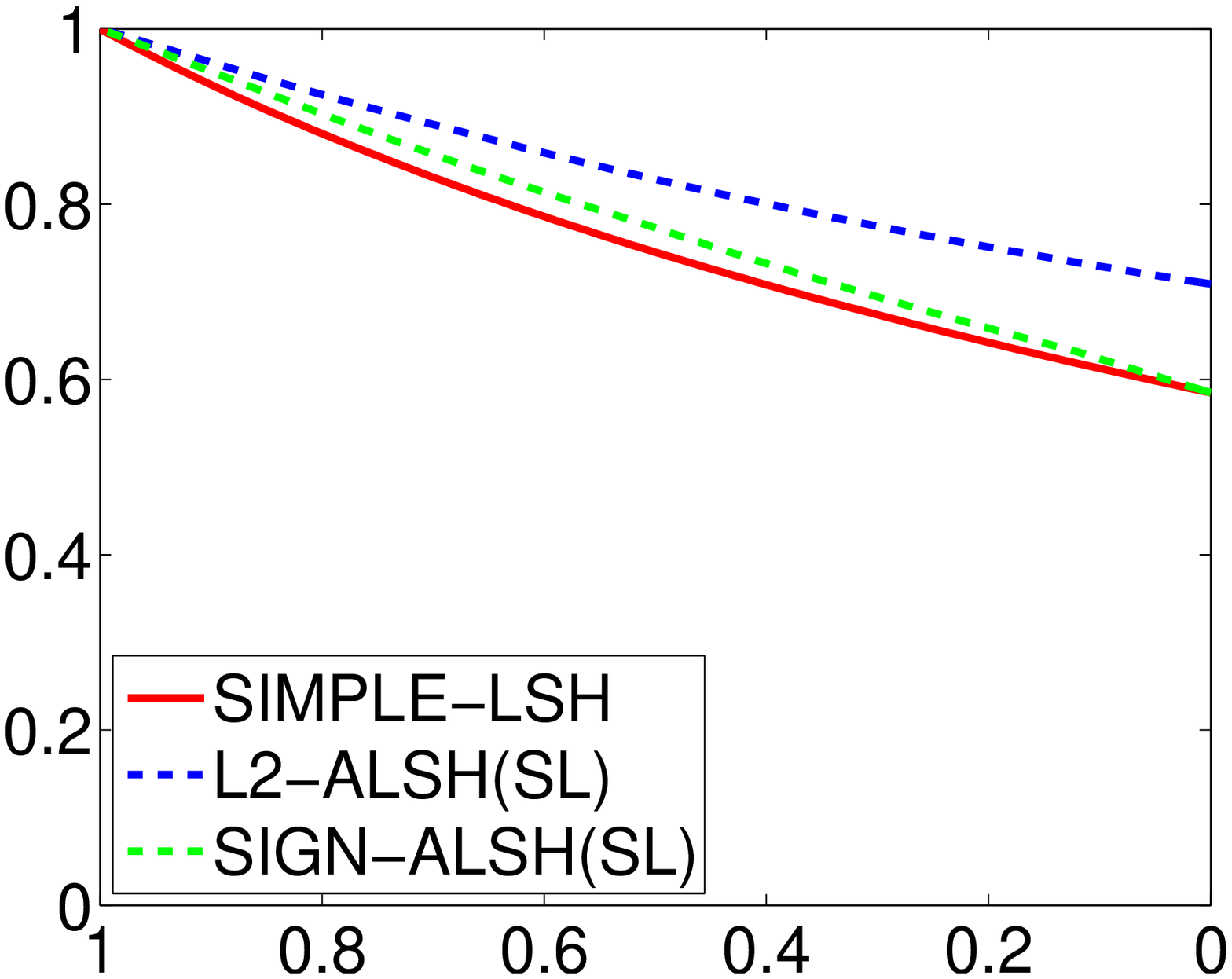}
\setlength{\epsfxsize}{0.31\textwidth}
\epsfbox{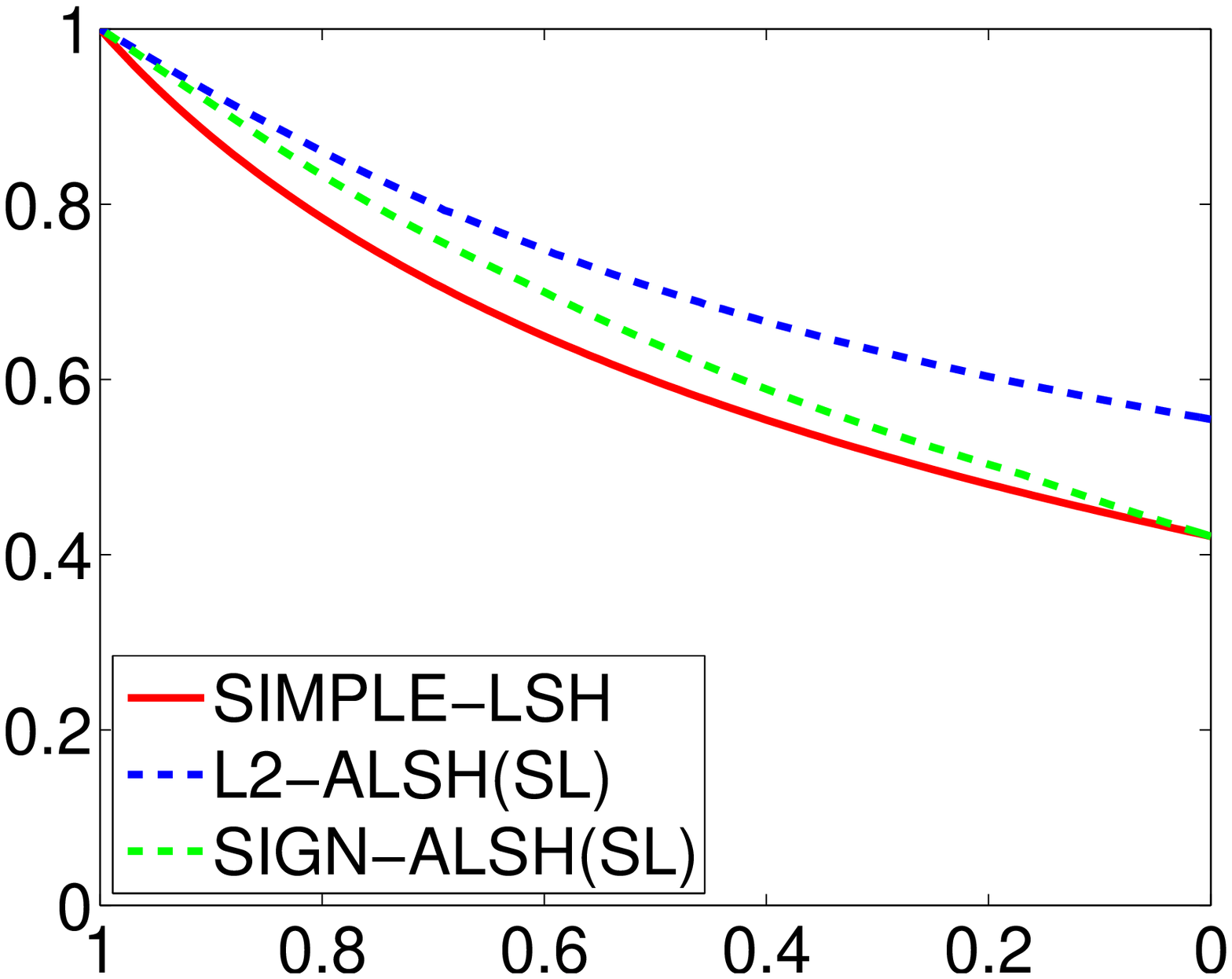}
}
\vspace{0.1in}
\hbox{ \centering 
\setlength{\epsfxsize}{0.332\textwidth}
\epsfbox{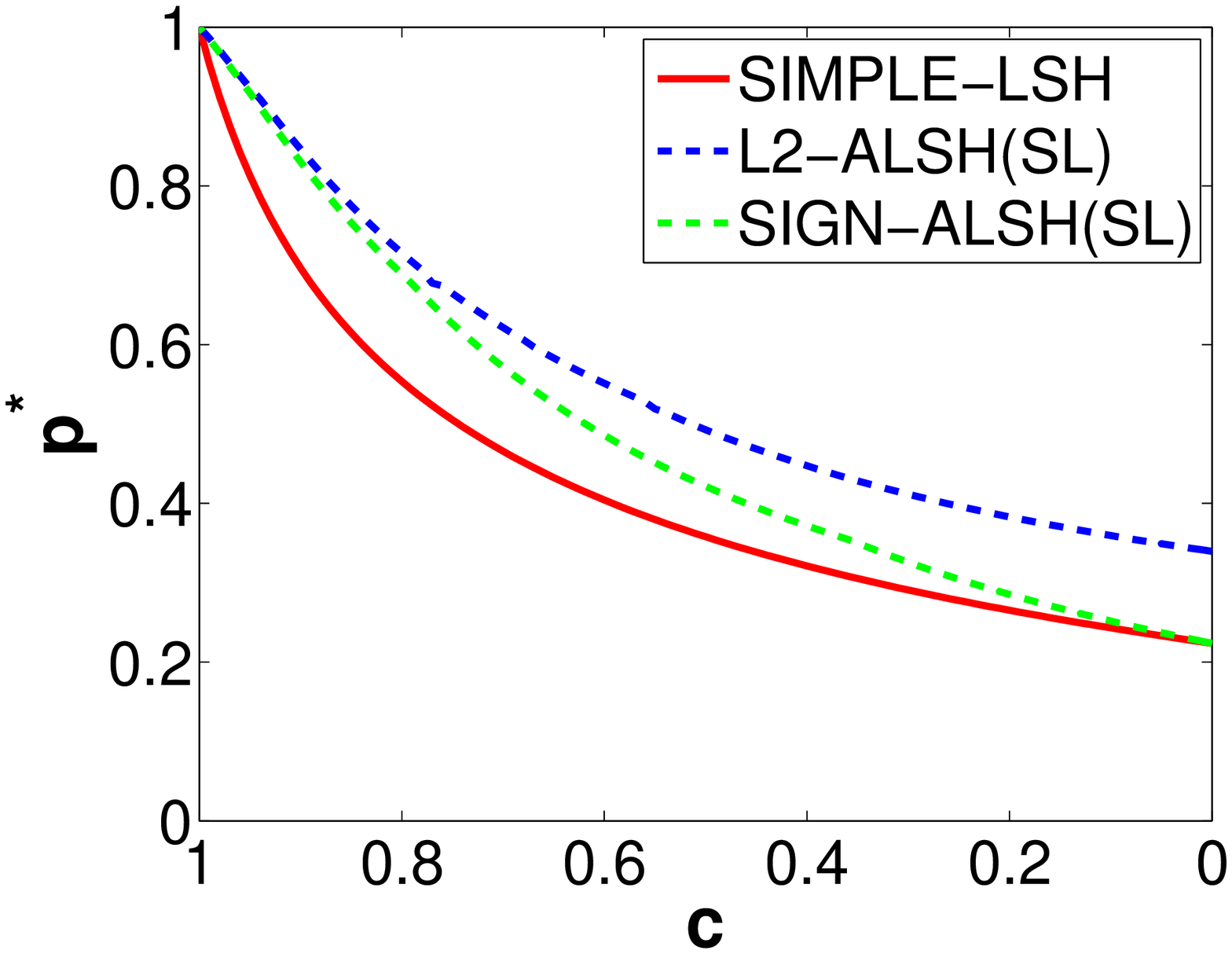}
\setlength{\epsfxsize}{0.31\textwidth}
\epsfbox{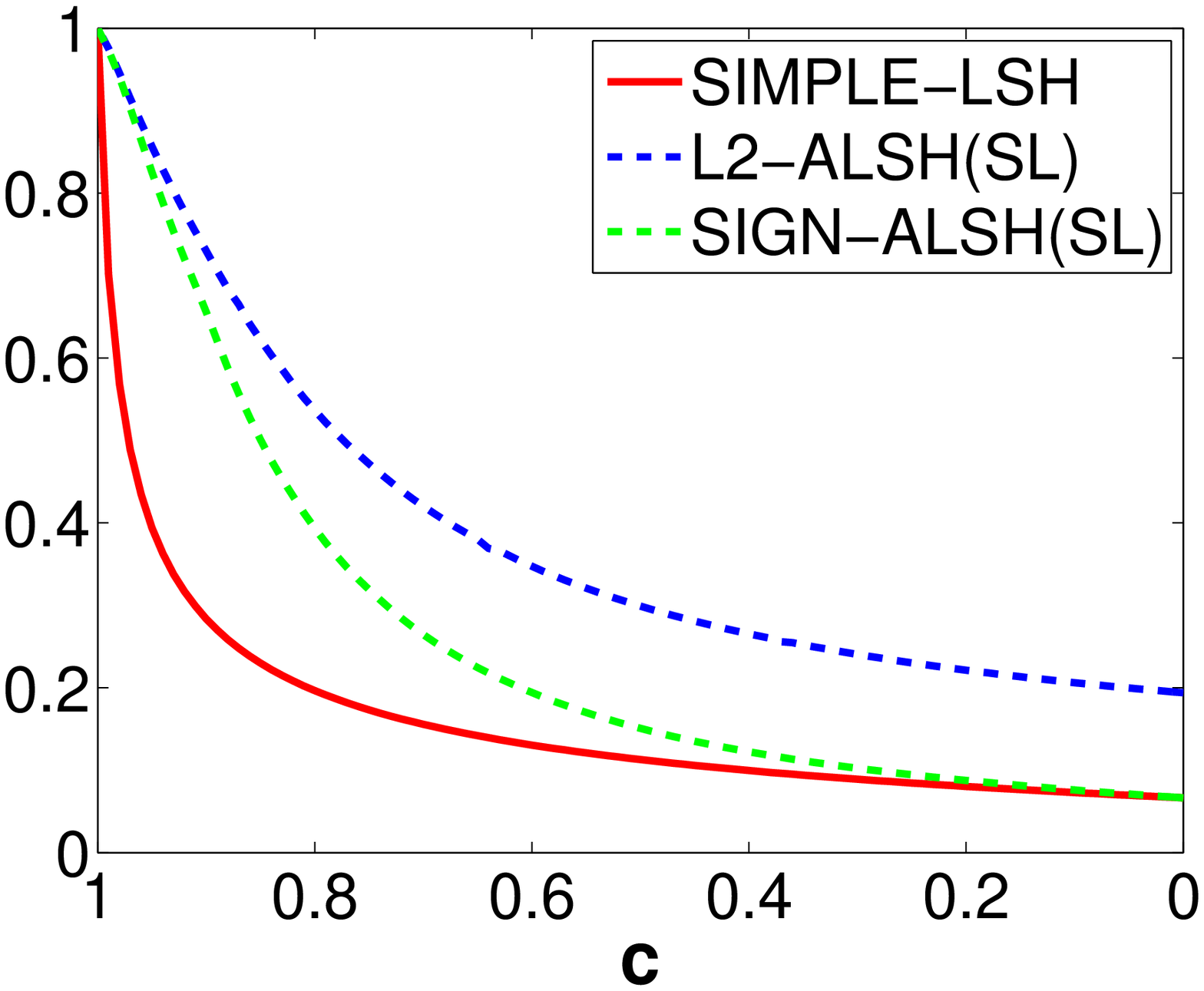}
\setlength{\epsfxsize}{0.31\textwidth}
\epsfbox{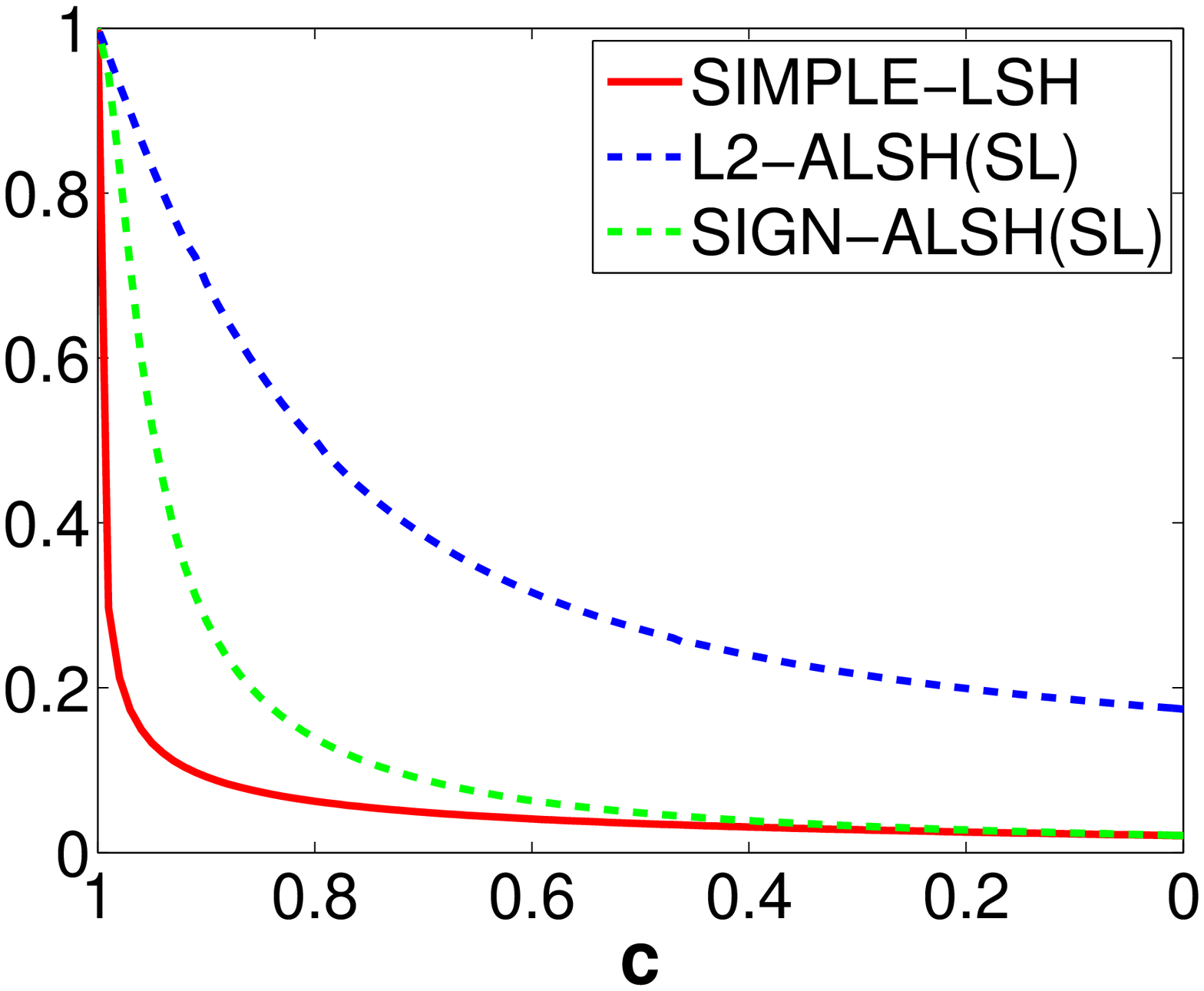}
}
\vspace{0.2in}
\hbox{ \centering 
}
\begin{picture}(0,0)(0,0)
\put(83, 292){$S=0.3$}
\put(235, 292){ $S=0.5$}
\put(394, 292){ $S=0.7$}
\put(80,163){ $S=0.9$}
\put(233,163){ $S=0.99$}
\put(388,163){ $S=0.999$}
\end{picture}
\vspace{-0.5in}
\caption{\small The optimal hashing quality $\rho^*$ for different hashes
(lower is better).  \label{fig:rho}
}
\end{figure*}

\begin{theorem}\label{thm:smiple-lsh}
  \textsc{simple-lsh} given in \eqref{eq:ha} is a universal LSH over
  $\XX_\bullet,\YY_\circ$.  That is, for every $0<S<1$ and $0<c<1$, it
  is an $(S,cS)$-LSH over $\XX_\bullet,\YY_\circ$.  Furthermore, it
  has hashing quality:
$$\rho=\tfrac{\log\bigg(1-\tfrac{\cos^{-1}(S)}{\pi}\bigg)}{\log\bigg(1-\tfrac{\cos^{-1}(cS)}{\pi}\bigg)}.$$
\end{theorem}
\begin{proof}
For any $x\in\XX_\bullet$ and $q\in\YY_\circ$ we have \citep{goemans95}:
\begin{align}
\PP[h_a(P(q))=h_a(P(x))] &=1-\frac{\cos^{-1}(q^\top x)}{\pi}.
\end{align}
Therefore:
\begin{itemize*}
\item if $q^\top x \geq S$, then
$$
\PP\big[h_{a}(P(q))=h_{a}(P(x))\big] \geq 1-\frac{\cos^{-1}(S)}{\pi}
$$
\item if $q^{\top}x \leq cS$, then
$$
\PP\big[h_{a}(P(q))=h_{a}(P(x))\big] \leq 1-\frac{\cos^{-1}(cS)}{\pi}
$$
\end{itemize*}
Since for any $0\leq x\leq1$, $1-\frac{\cos^{-1}(x)}{\pi}$ is a monotonically increasing function, this gives us an LSH.
\end{proof}

\subsection{Theoretical Comparison}\label{sec:theory}
Earlier we discussed that an LSH with the smallest possible hashing
quality $\rho$ is desirable.  In this Section, we compare the best
achievable hashing quality and show that \textsc{simple-lsh} allows
for much better hashing quality compared to \textsc{l2-alsh(sl)}, as
well as compared to the improved hash \textsc{sign-lsh(sl)}.

For \textsc{l2-alsh(sl)} and \textsc{sign-alsh(sl)}, for each desired
threshold $S$ and ratio $c$, one can optimize over the parameters $m$
and $U$, and for \textsc{l2-alsh(sl)} also $r$, to find the hash with
the best $\rho$.  This is a non-convex optimization problem and
\citet{shrivastava14} suggest using grid search to find a bound
on the optimal $\rho$.  We followed the procedure, and grid,
as suggested by \citet{shrivastava14}\footnote{We actually used a slightly tighter bound---a
  careful analysis shows the denominator in equation 19 of \citet{shrivastava14}
  can be $\log F_r(\sqrt{1+m/2-2cSU + (cSU)^{2^{m+1}})})$}.  For \textsc{simple-lsh}
no parameters need to be tuned, and for each $S,c$ the hashing quality
is given by Theorem \ref{thm:smiple-alsh}.  In Figure \ref{fig:rho} we
compare the optimal hashing quality $\rho$ for the three methods, for
different values of $S$ and $c$.  It is clear that the
\textsc{simple-lsh} dominates the other methods. \removed{ For
  thresholds $S$ far away from 1 (the typically interesting regime),
  even though the standard \textsc{$L_2$-lsh} is a valid LSH,
  \textsc{l2-alsh(sl)} indeed yields significantly better hash quality
  (though \textsc{simple-lsh} dominates both).}

\removed{For \textsc{$L_2$-lsh} the parameters $U$ and $r$ need to be optimized
to obtain the minimal $\rho^*_{\textsc{$L_2$-LSH}}$ for each choice of
$S$ and $c$. We solve the following optimization by a grid search over
$U$ and $r$:
\begin{eqnarray}
&&\rho^*_{\textsc{$L_2$-lsh}}(S,c)=\\
\nonumber
&&\min_{0<r \atop 0<U}\frac{\log \FF_r\big(\sqrt{1-2SU+U^{2}}\big)}{\log \FF_r\big(1-cSU \big)}\\
\nonumber
&&\;\;\text{s.t.} \;\;\;\;U\leq 2S(1-c).
\nonumber
\end{eqnarray}}
\begin{figure*}[t!]
\hbox{ \centering \hspace{-0.2in}\
\setlength{\epsfxsize}{0.25\textwidth}
\epsfbox{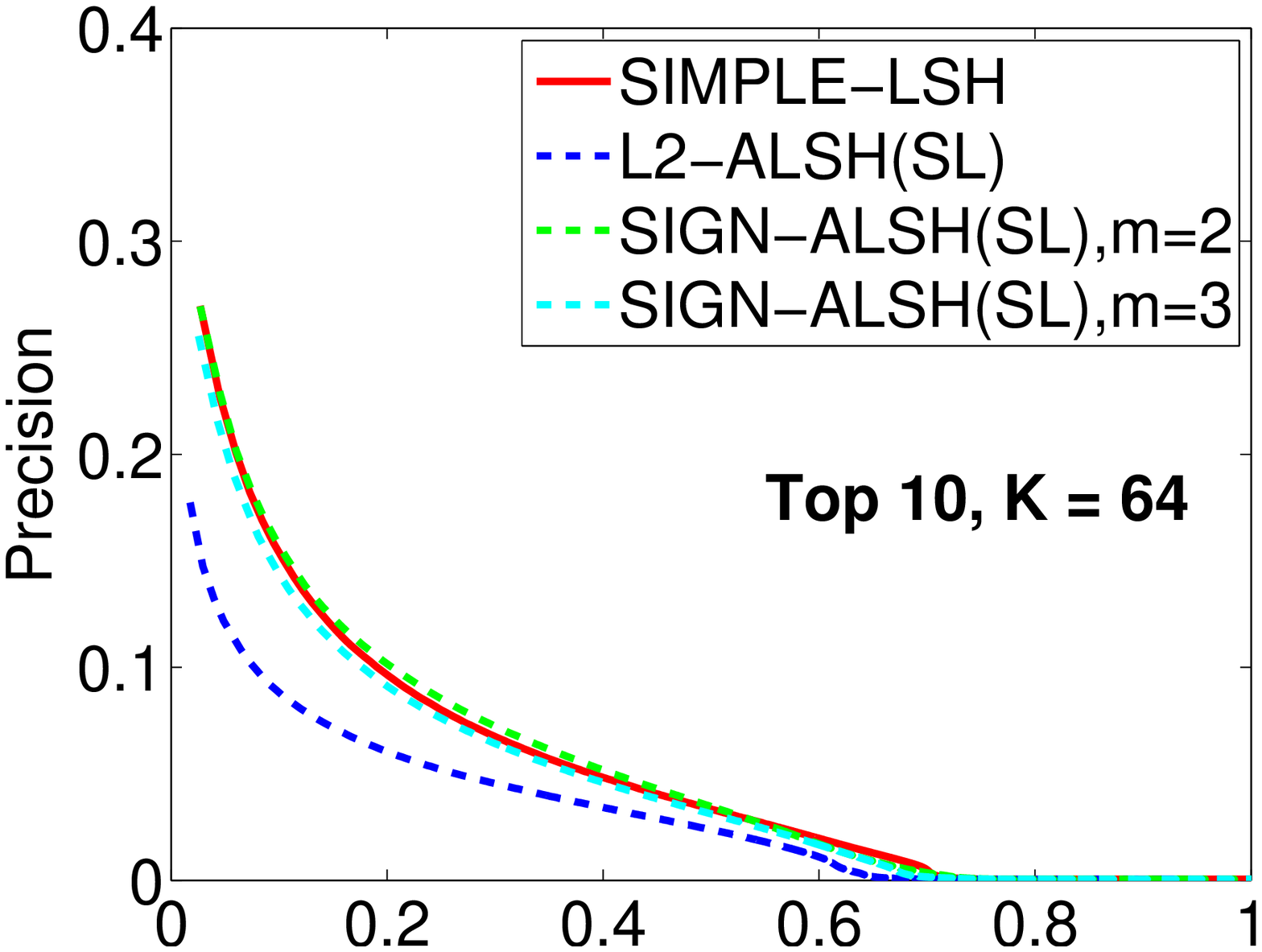}
\setlength{\epsfxsize}{0.235\textwidth}
\epsfbox{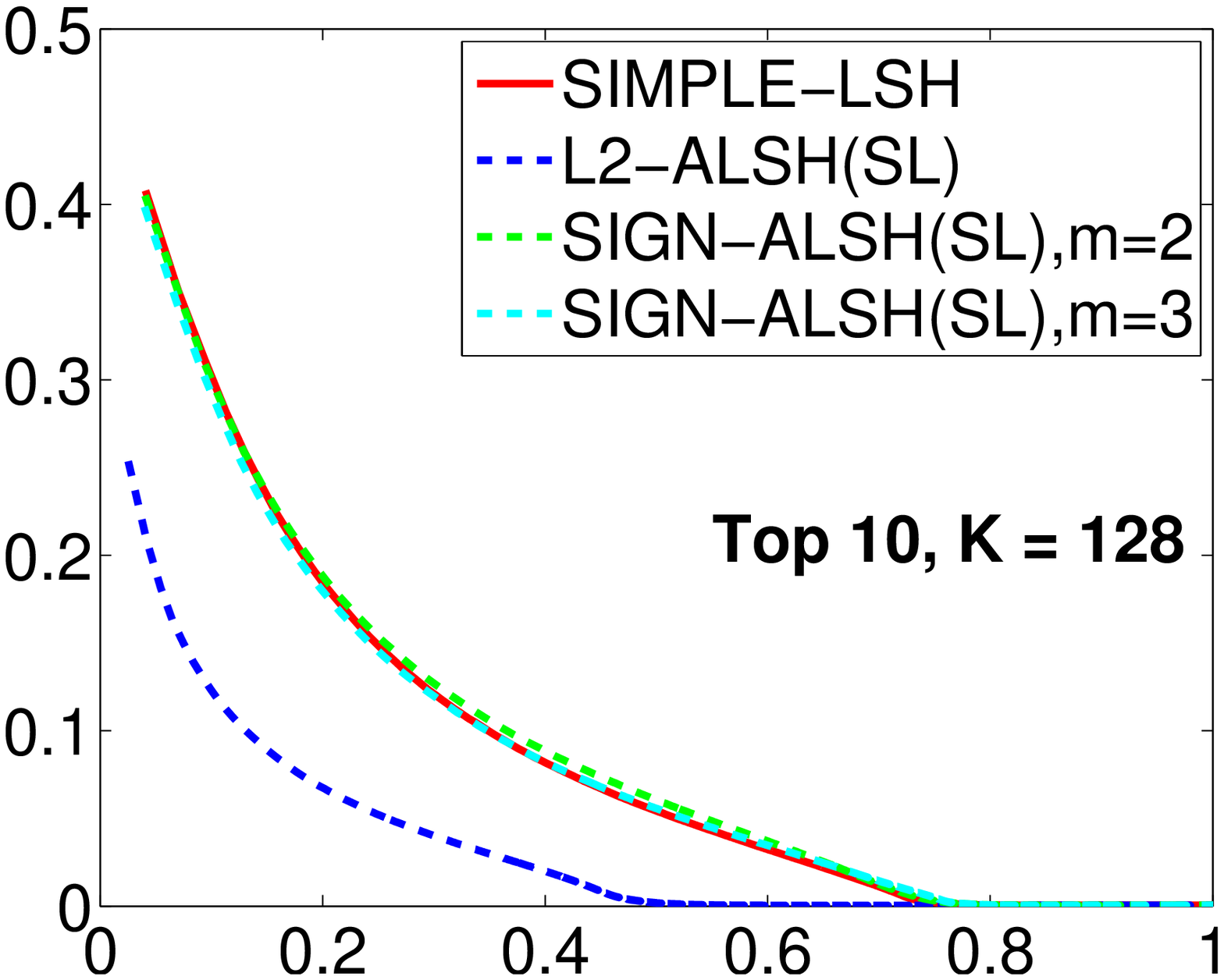}
\setlength{\epsfxsize}{0.235\textwidth}
\epsfbox{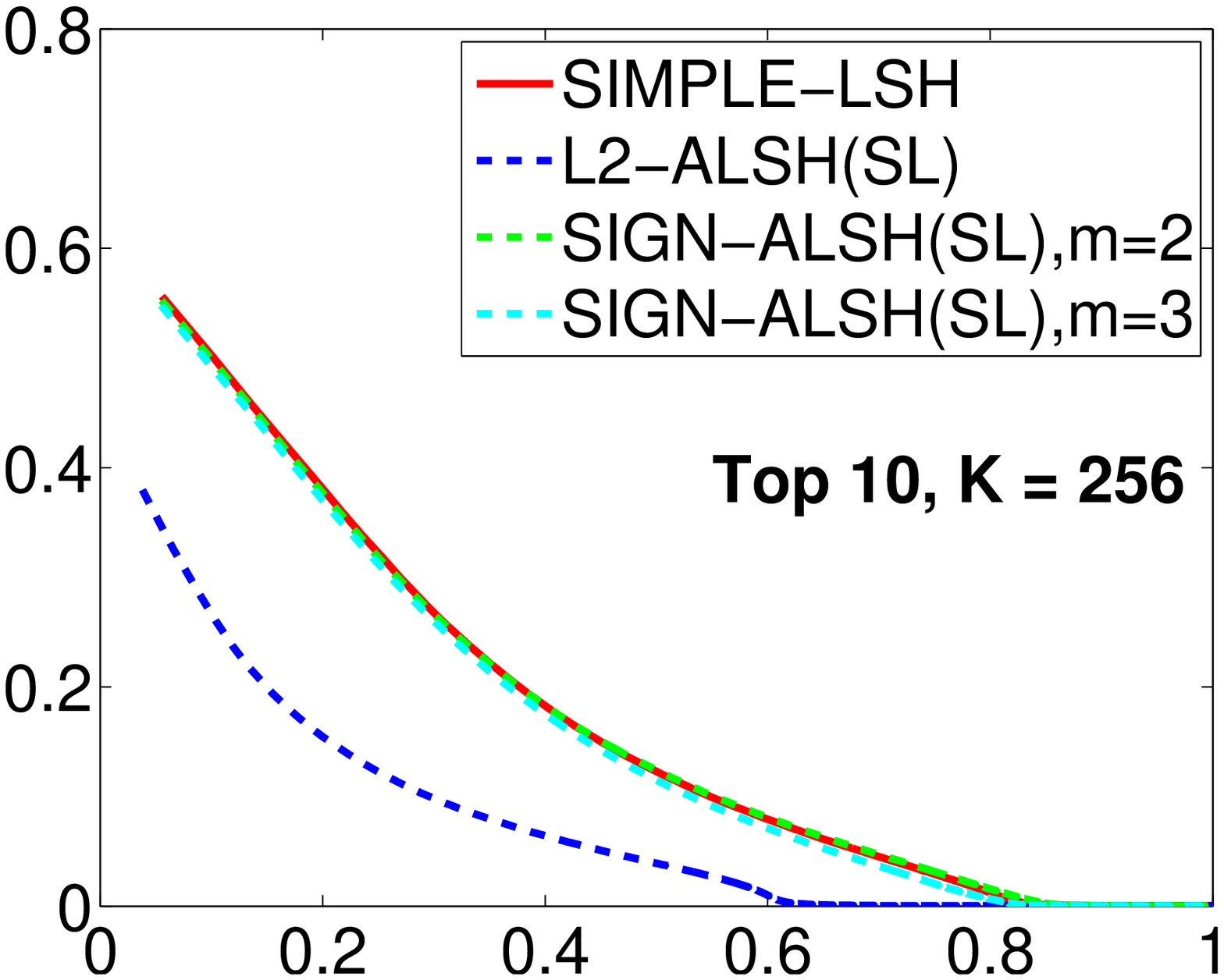}
\setlength{\epsfxsize}{0.235\textwidth}
\epsfbox{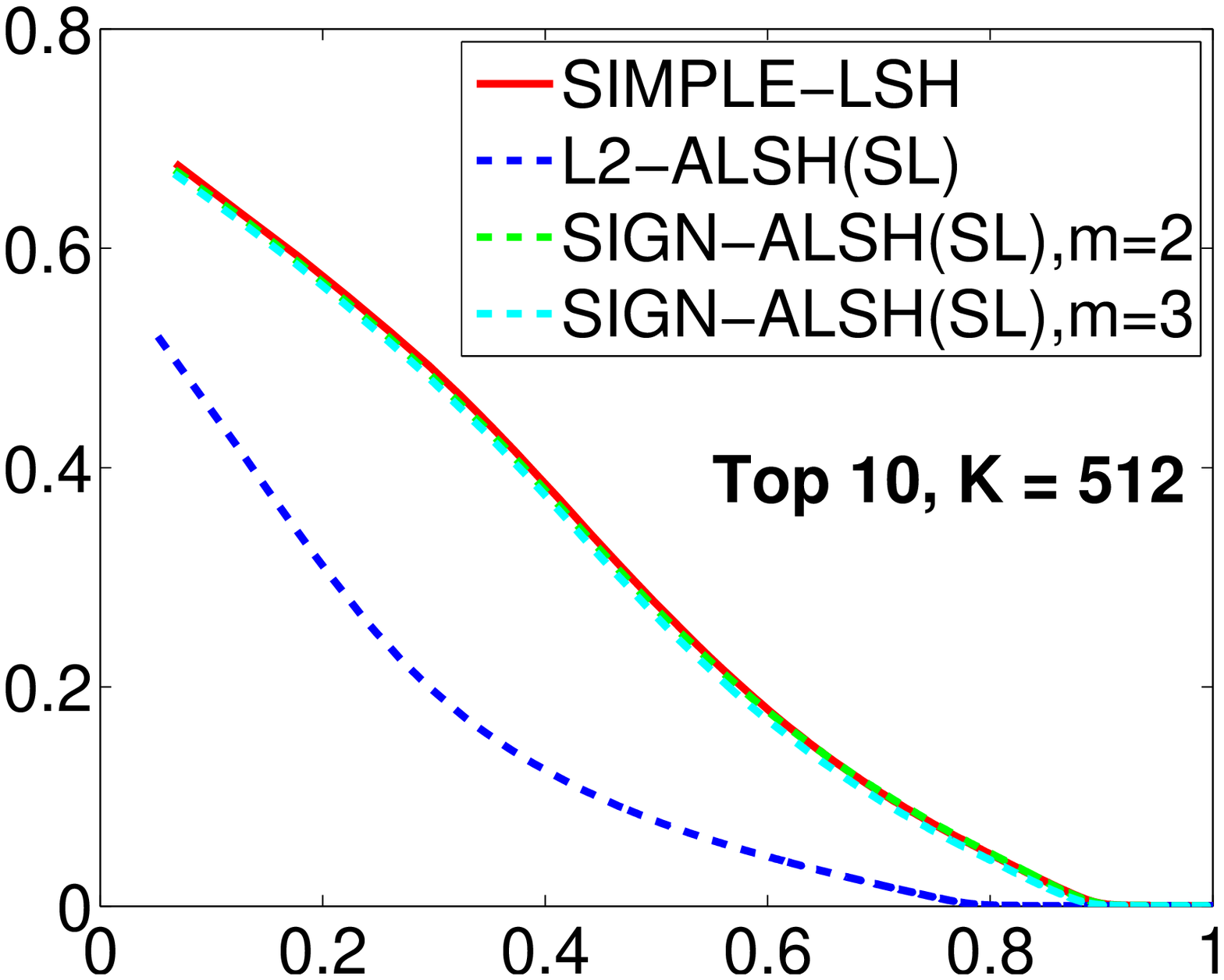}
}
\vspace{0.1in}
\hbox{ \centering \hspace{-0.2in}
\setlength{\epsfxsize}{0.25\textwidth}
\epsfbox{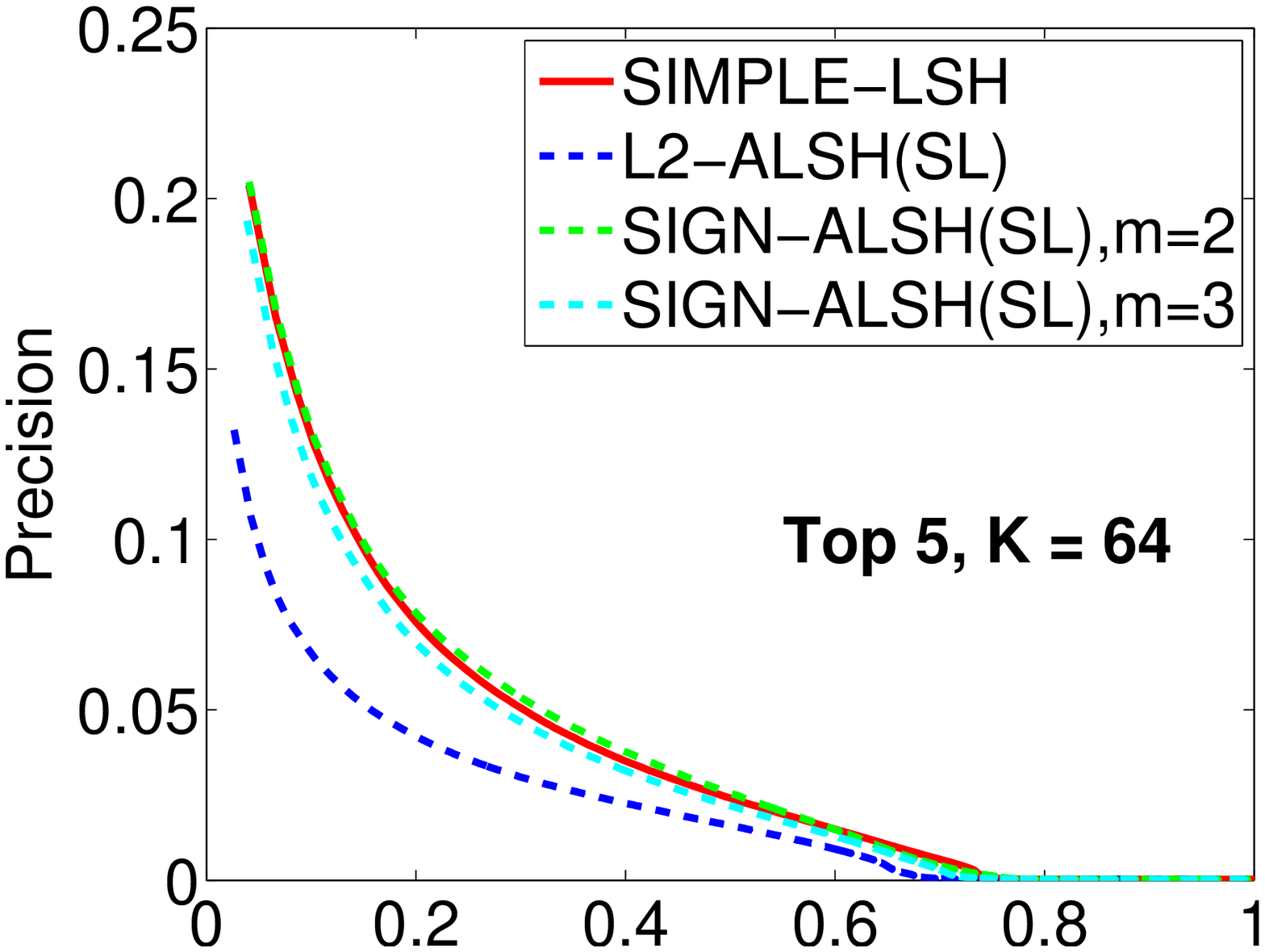}
\setlength{\epsfxsize}{0.235\textwidth}
\epsfbox{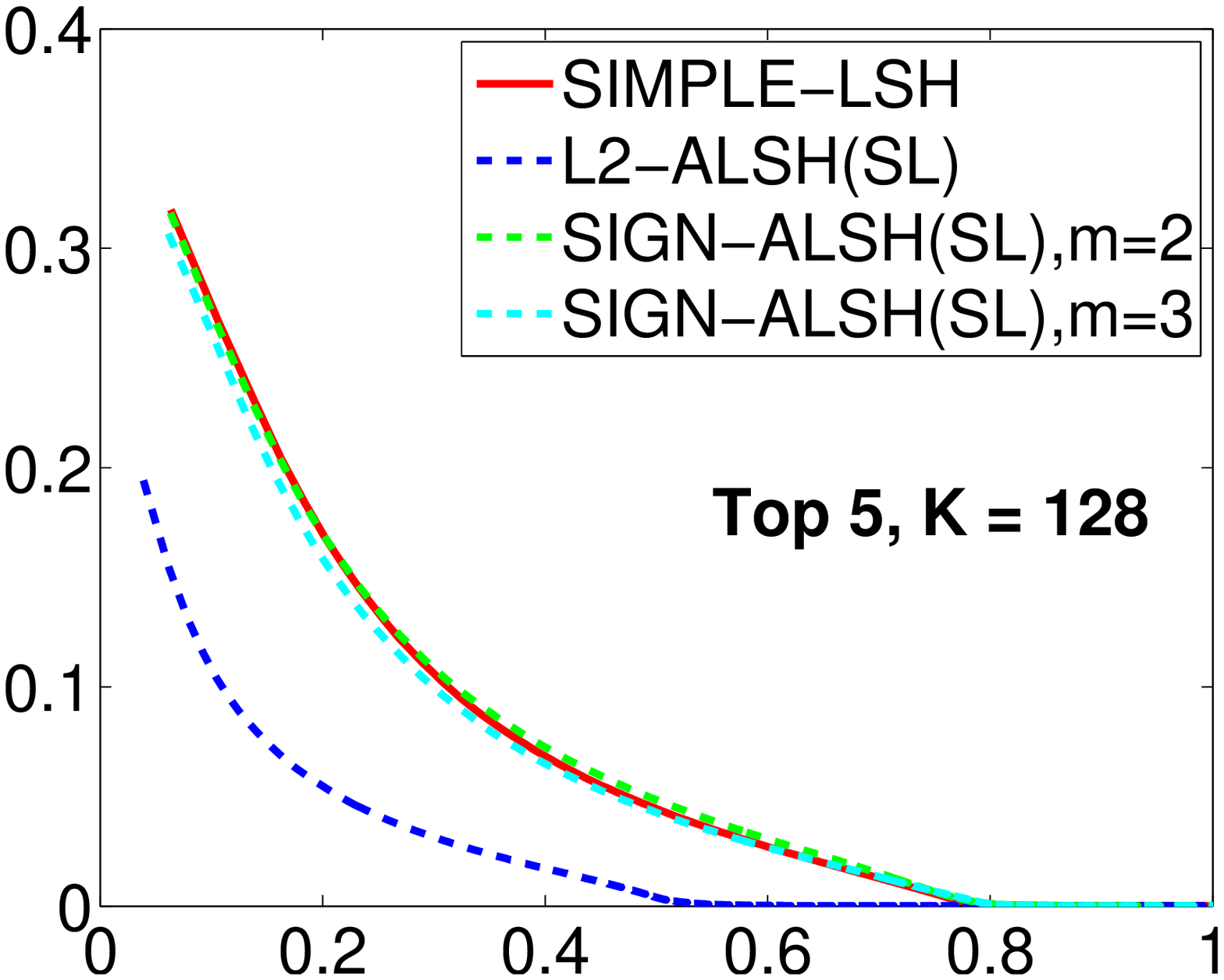}
\setlength{\epsfxsize}{0.235\textwidth}
\epsfbox{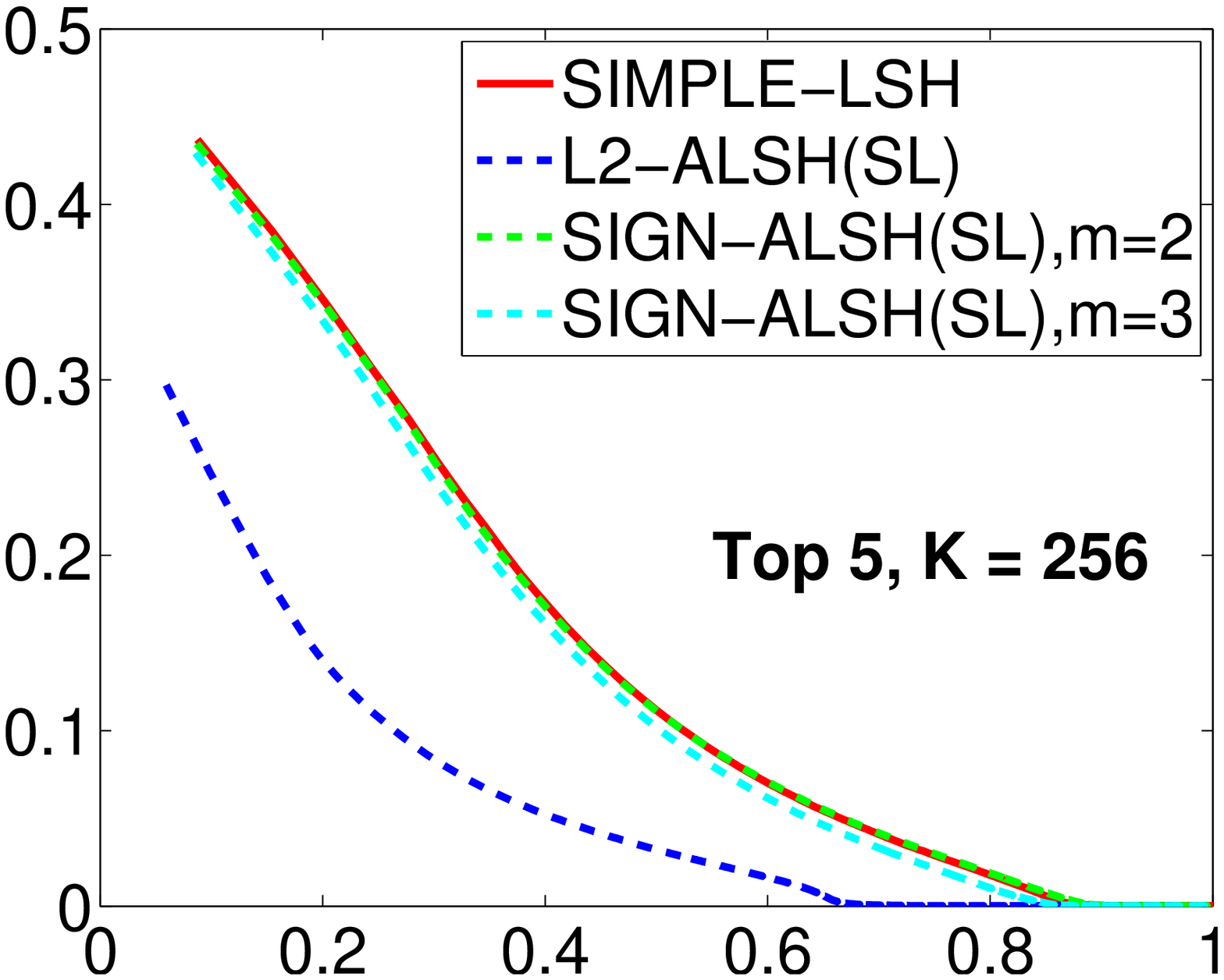}
\setlength{\epsfxsize}{0.235\textwidth}
\epsfbox{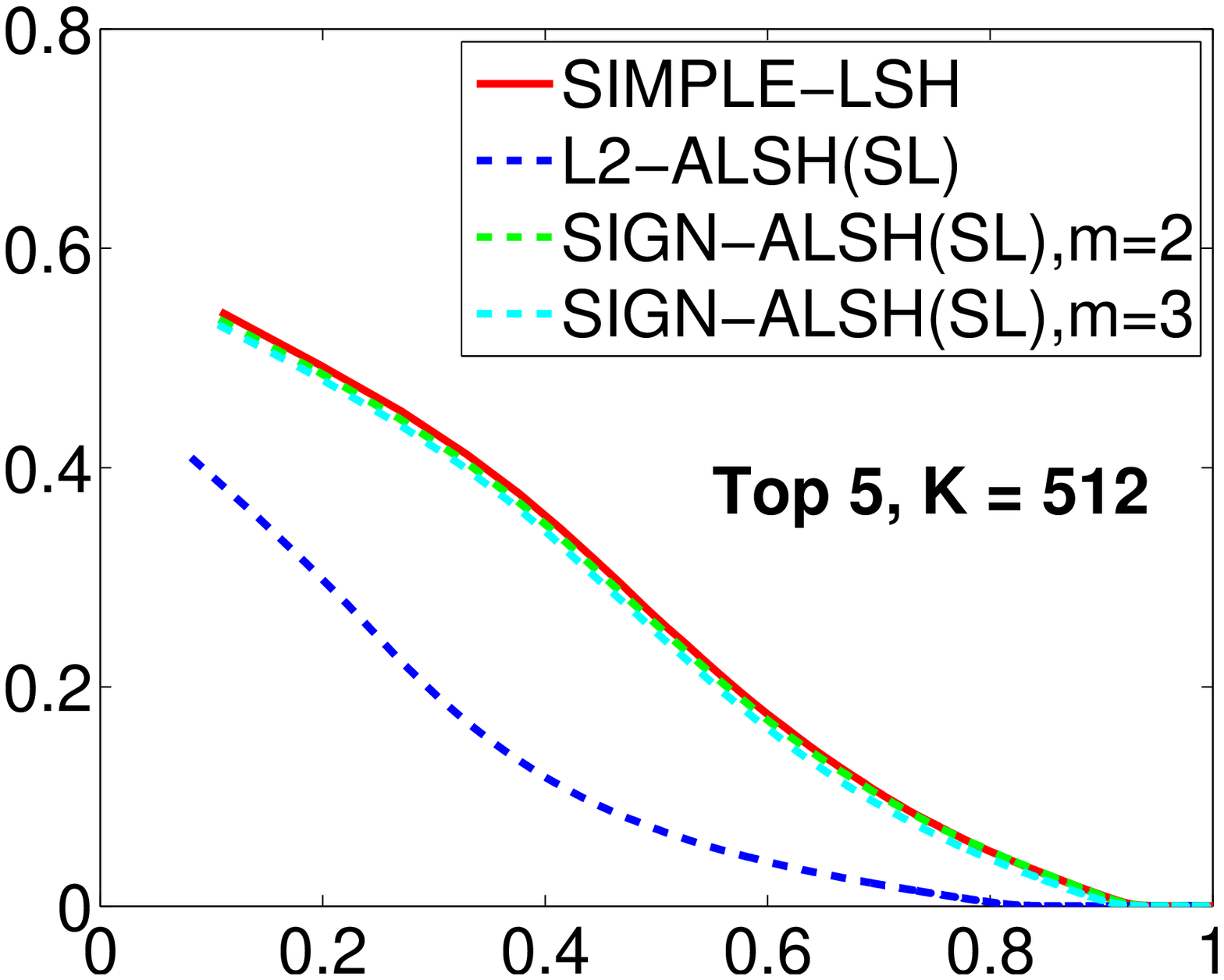}
}
\vspace{0.1in}
\hbox{ \centering \hspace{-0.2in}
\setlength{\epsfxsize}{0.25\textwidth}
\epsfbox{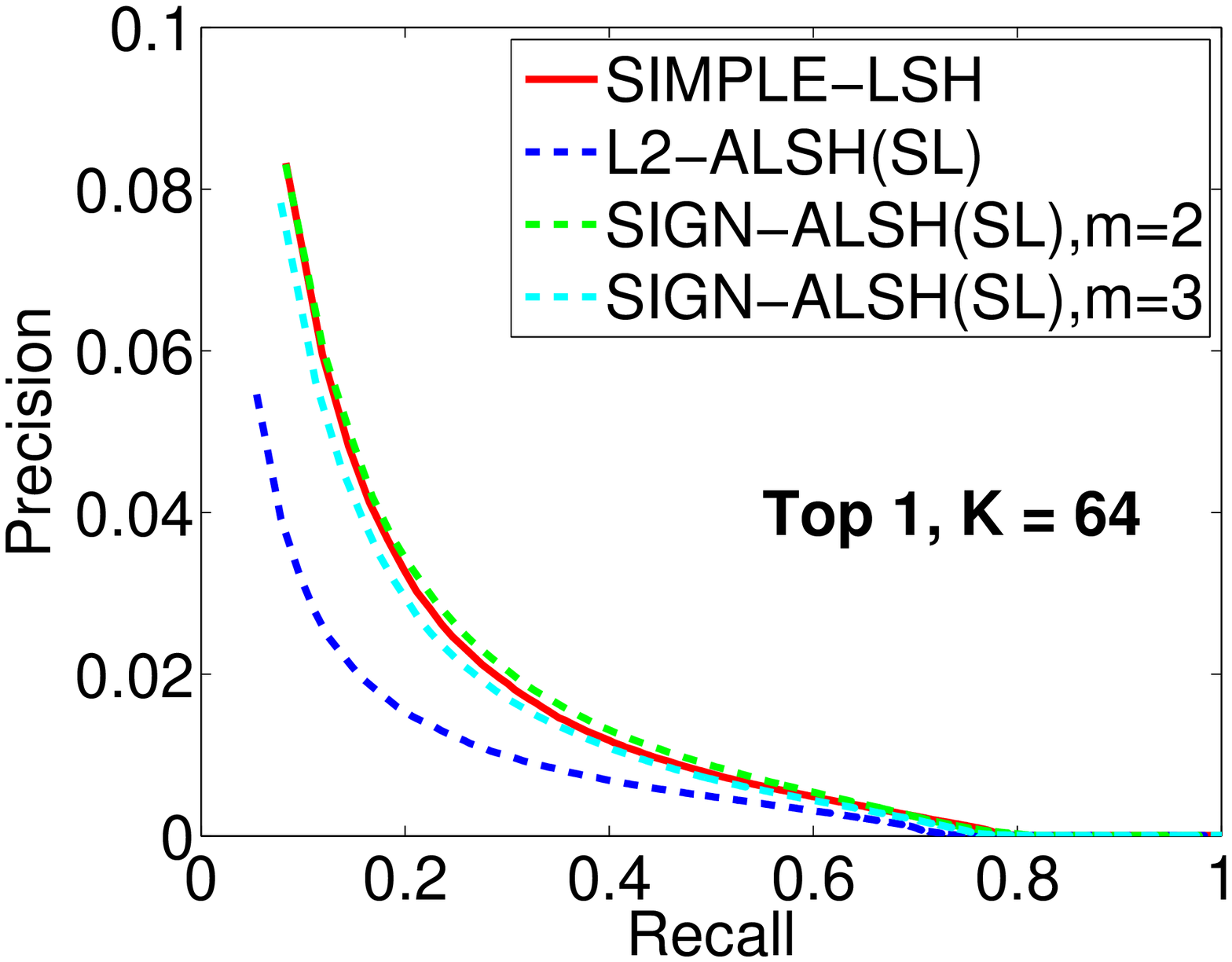}
\setlength{\epsfxsize}{0.235\textwidth}
\epsfbox{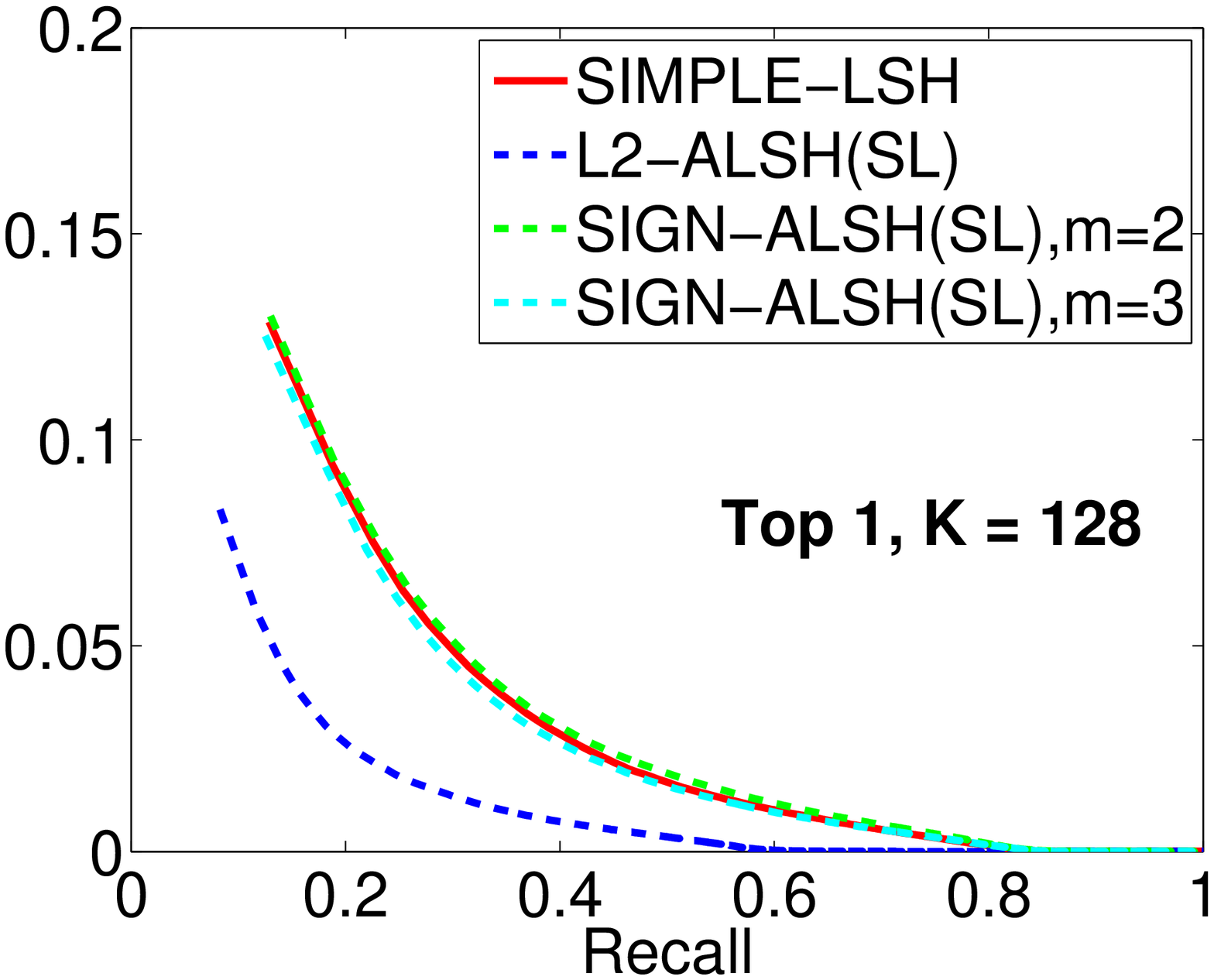}
\setlength{\epsfxsize}{0.235\textwidth}
\epsfbox{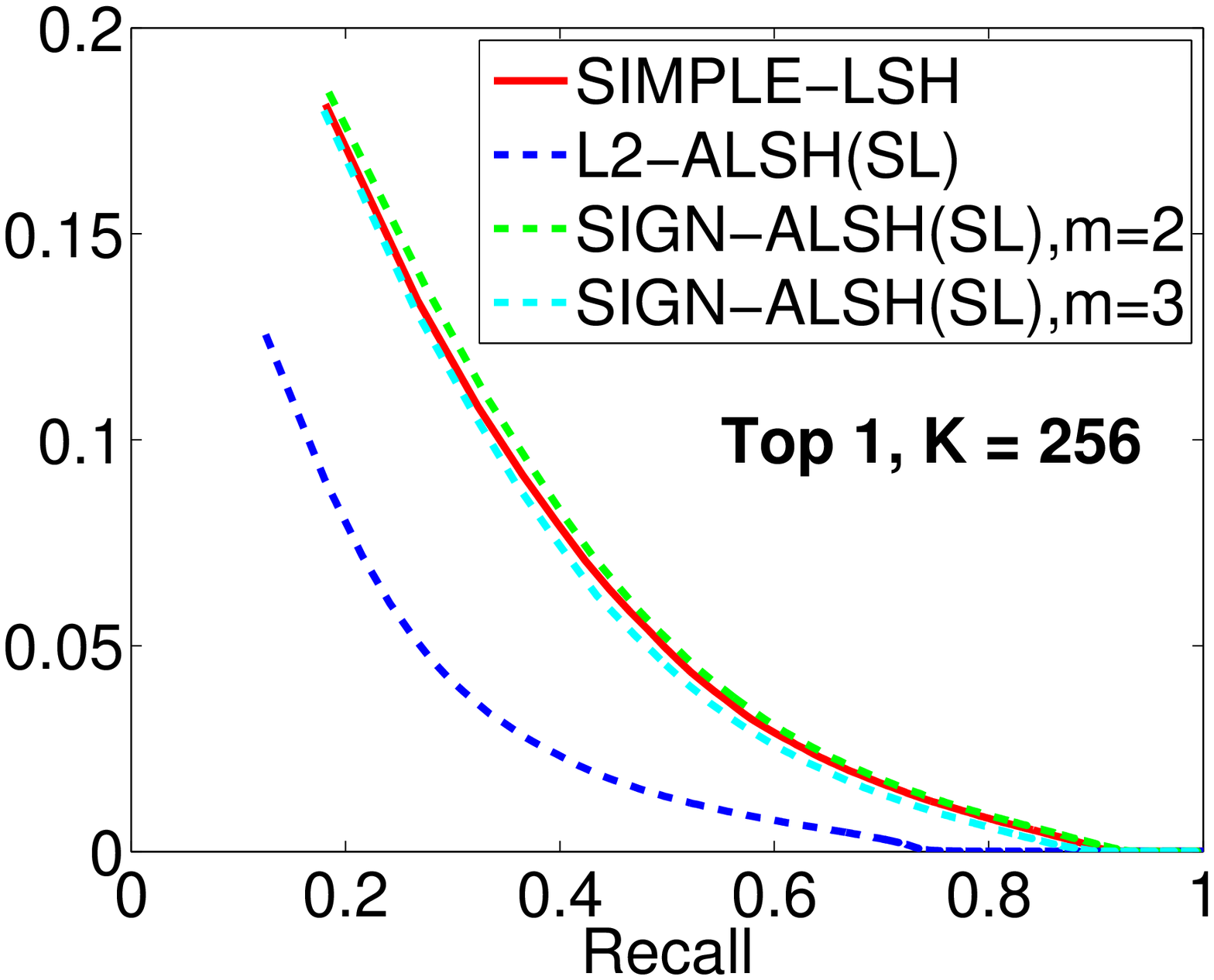}
\setlength{\epsfxsize}{0.235\textwidth}
\epsfbox{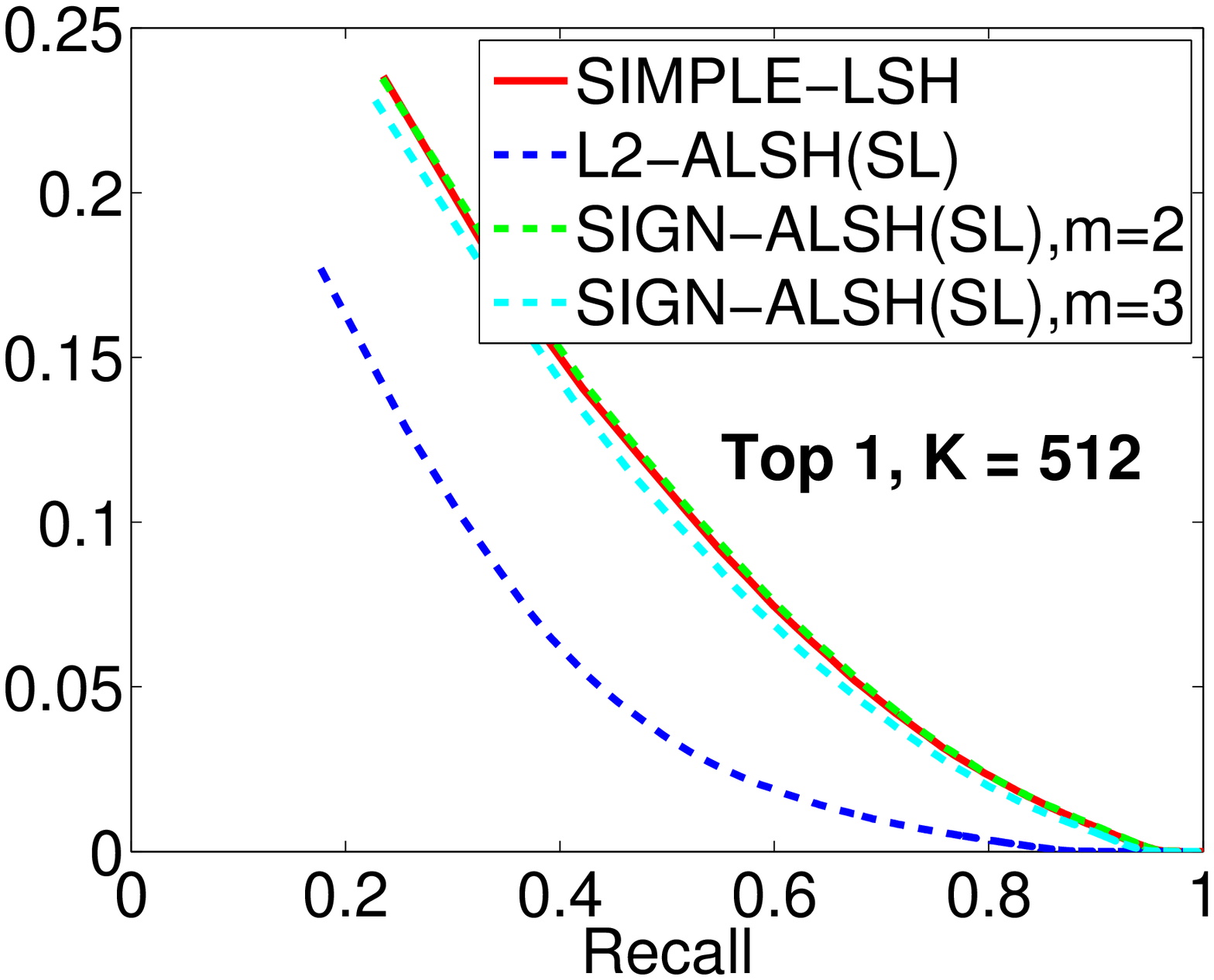}
}
\vspace{0.1in}
\caption{\small\textbf{Netflix}: Precision-Recall curves (higher is better)
  of retrieving top $T$ items by hash code of length $K$. \textsc{simple-lsh}
  is parameter-free.  For \textsc{l2-alsh(sl)}, we fix the parameters
  $m=3$, $U=0.84$, $r=2.5$ and
  for \textsc{sign-alsh(sl)} we used two
different settings of the parameters: $m=2$, $U=0.75$ and $m=3$, $U=0.85$. \label{fig:netflix}
}
\end{figure*}
\begin{figure*}[t!]
\vspace{0.1in}
\hbox{ \centering \hspace{-0.2in}\
\setlength{\epsfxsize}{0.25\textwidth}
\epsfbox{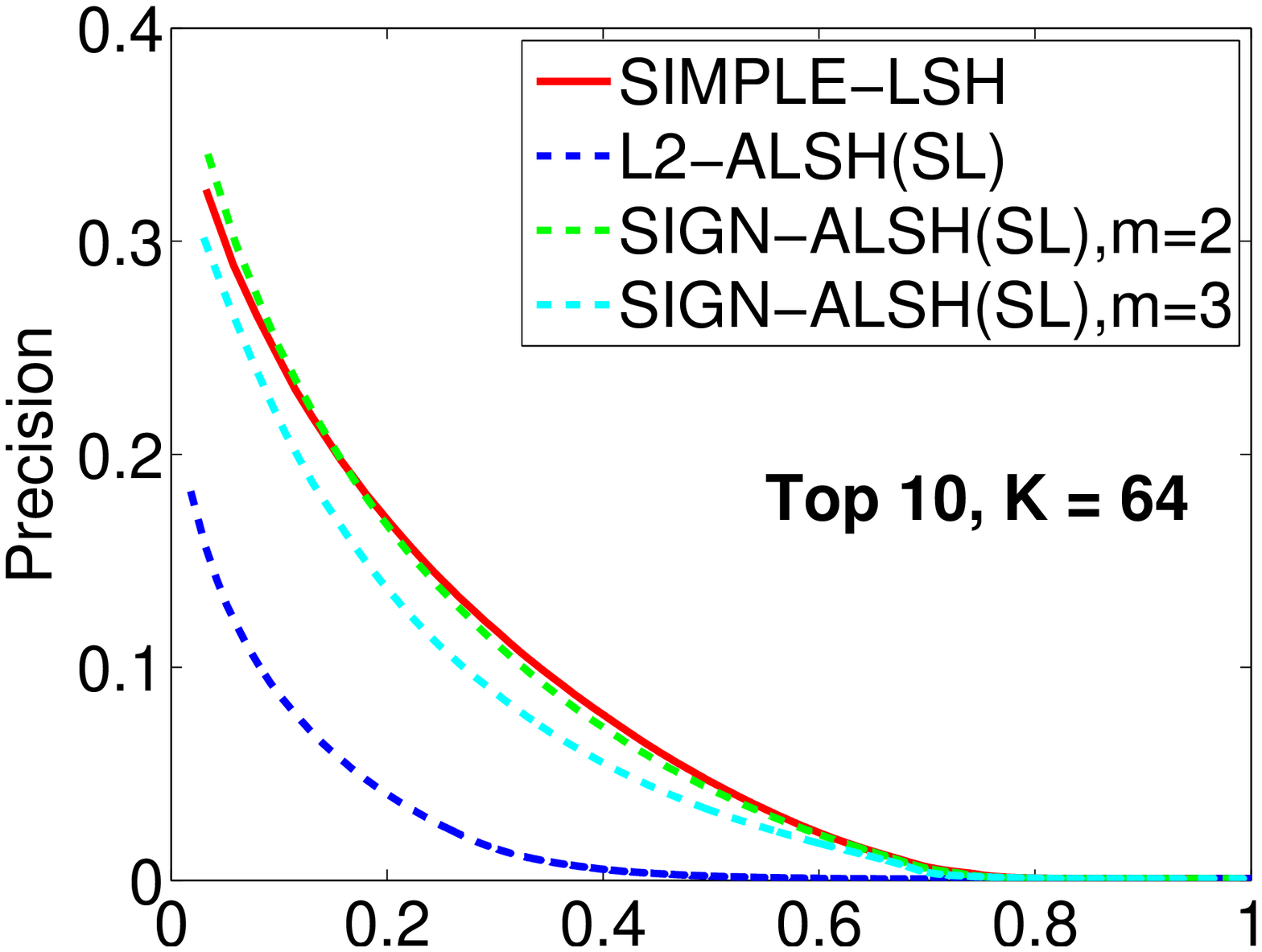}
\setlength{\epsfxsize}{0.235\textwidth}
\epsfbox{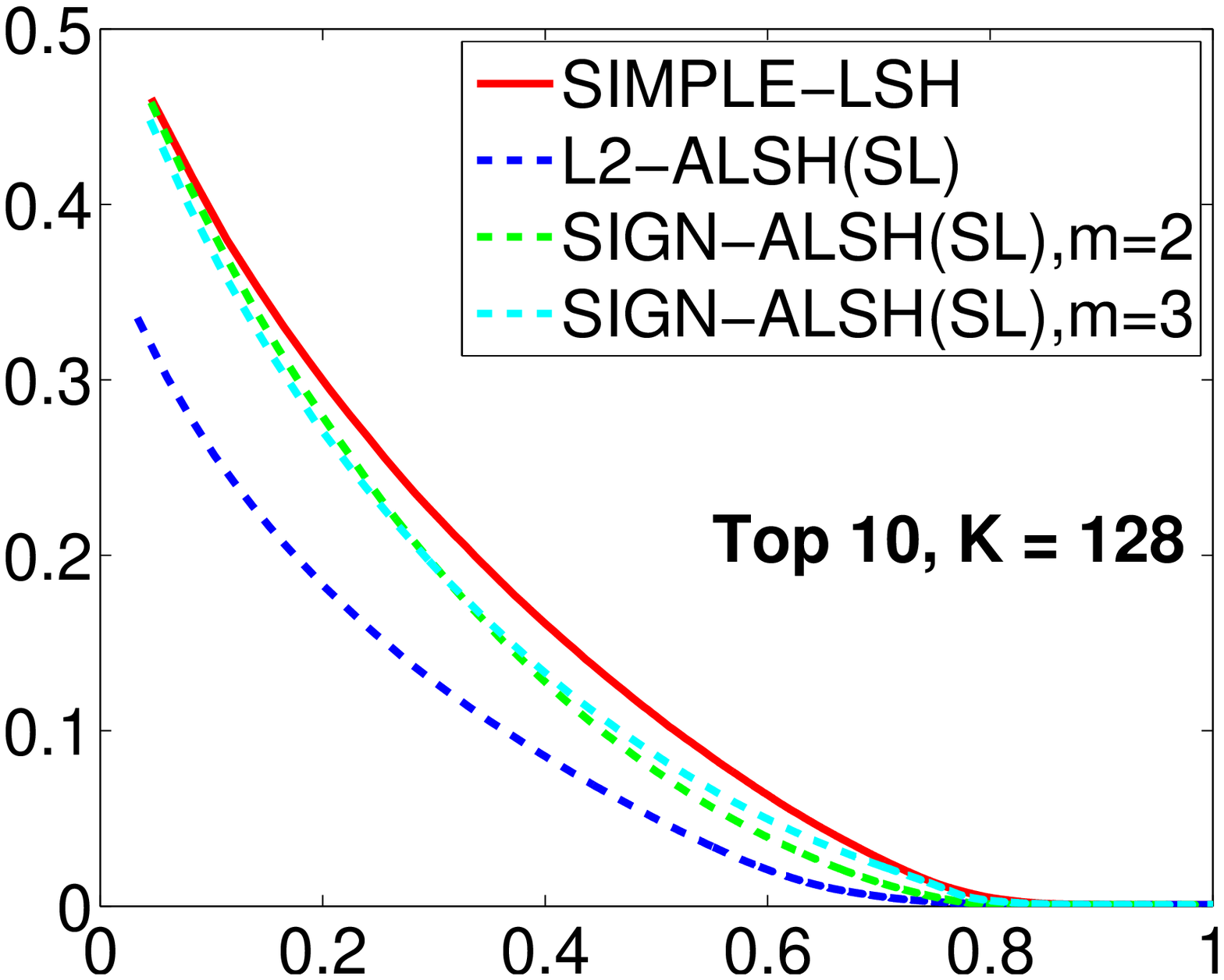}
\setlength{\epsfxsize}{0.235\textwidth}
\epsfbox{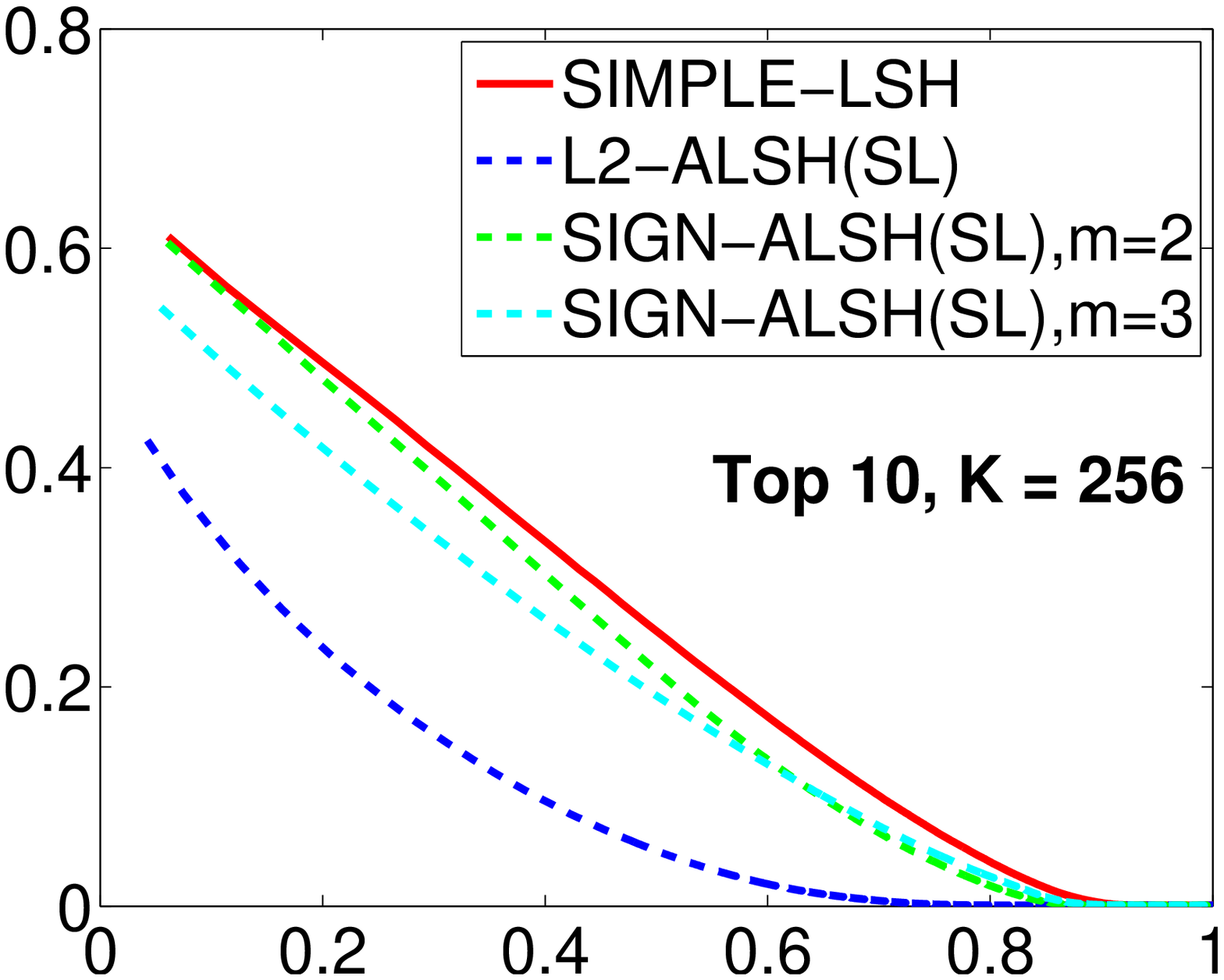}
\setlength{\epsfxsize}{0.235\textwidth}
\epsfbox{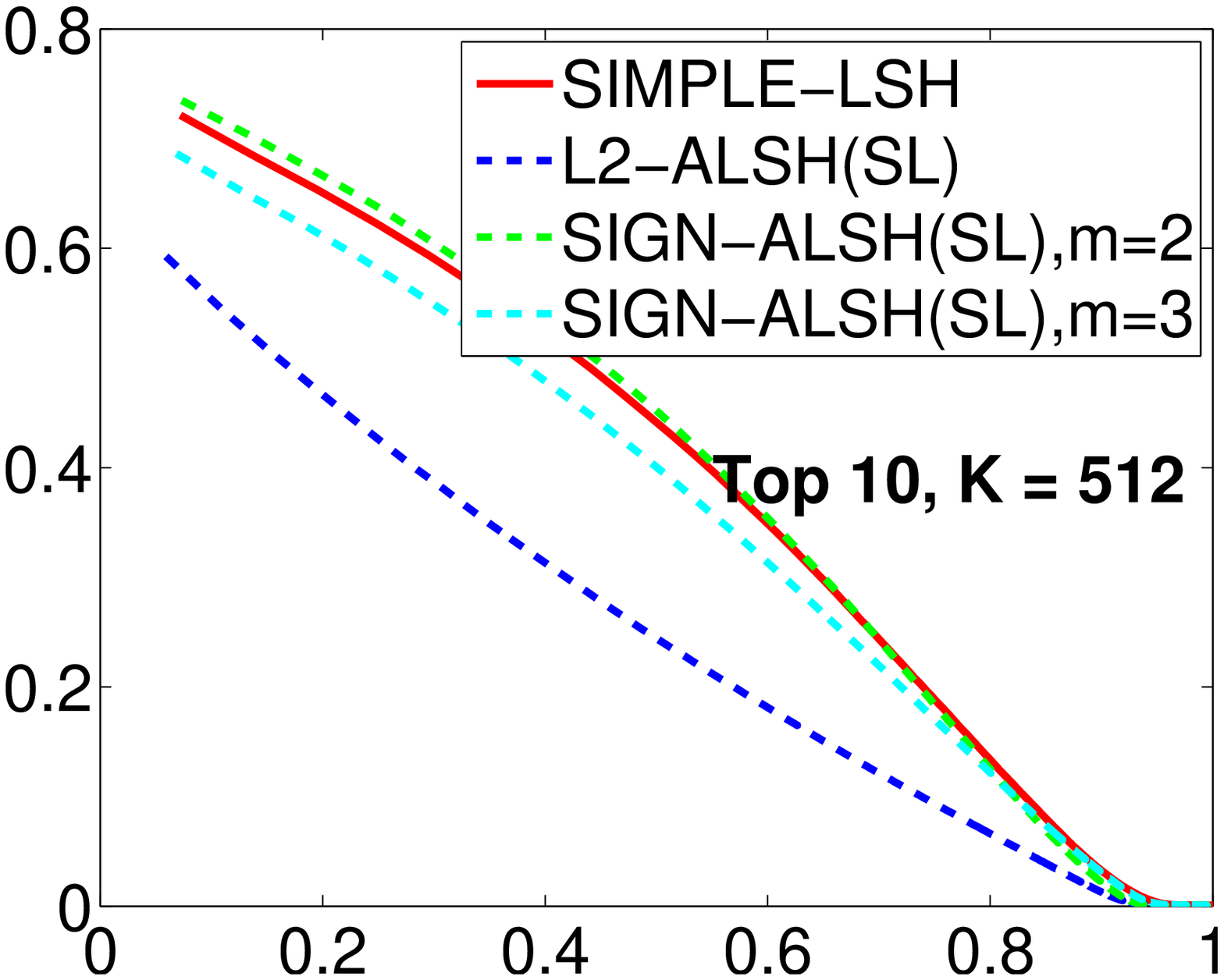}
}
\vspace{0.1in}
\hbox{ \centering \hspace{-0.2in}
\setlength{\epsfxsize}{0.25\textwidth}
\epsfbox{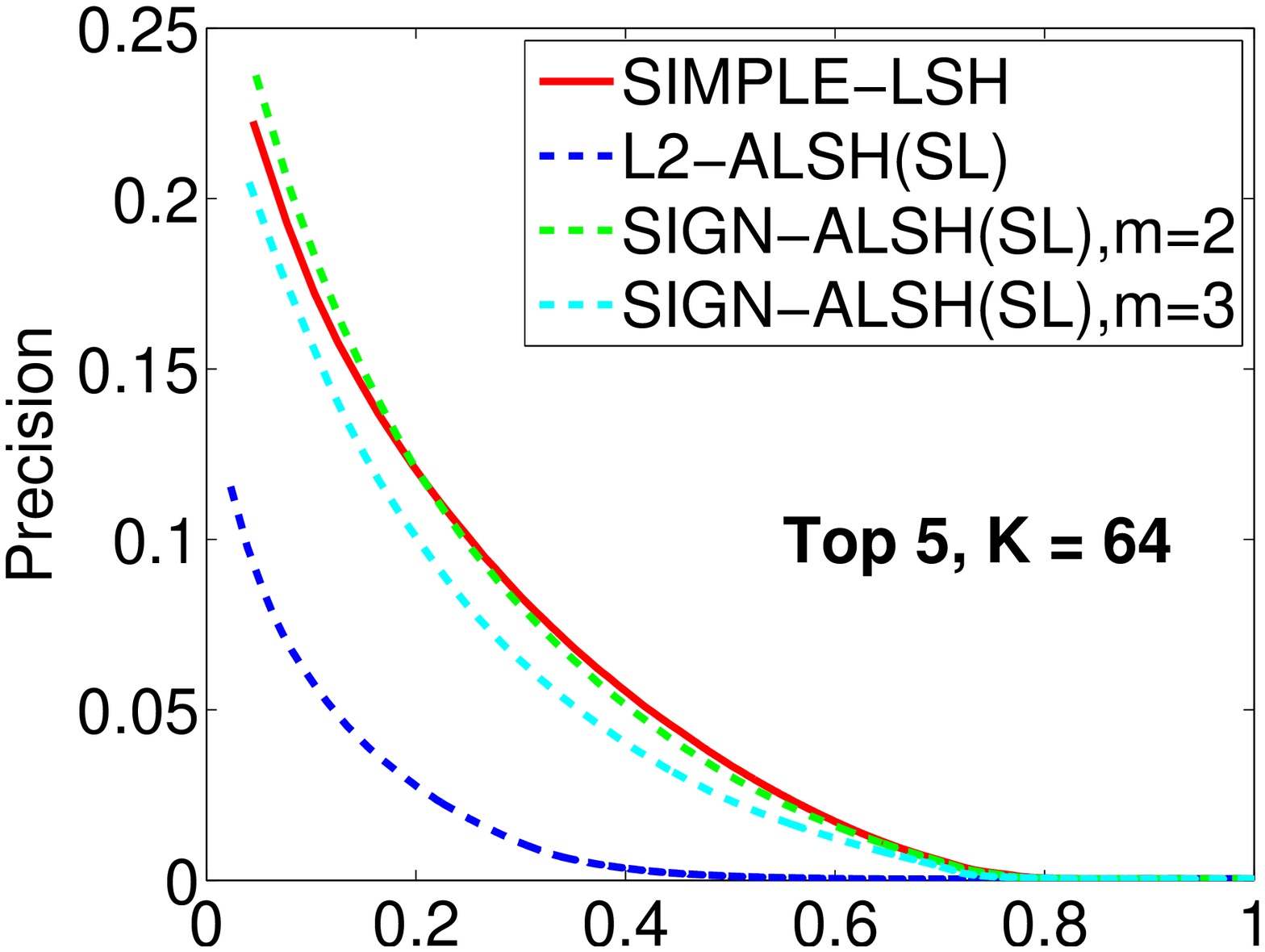}
\setlength{\epsfxsize}{0.235\textwidth}
\epsfbox{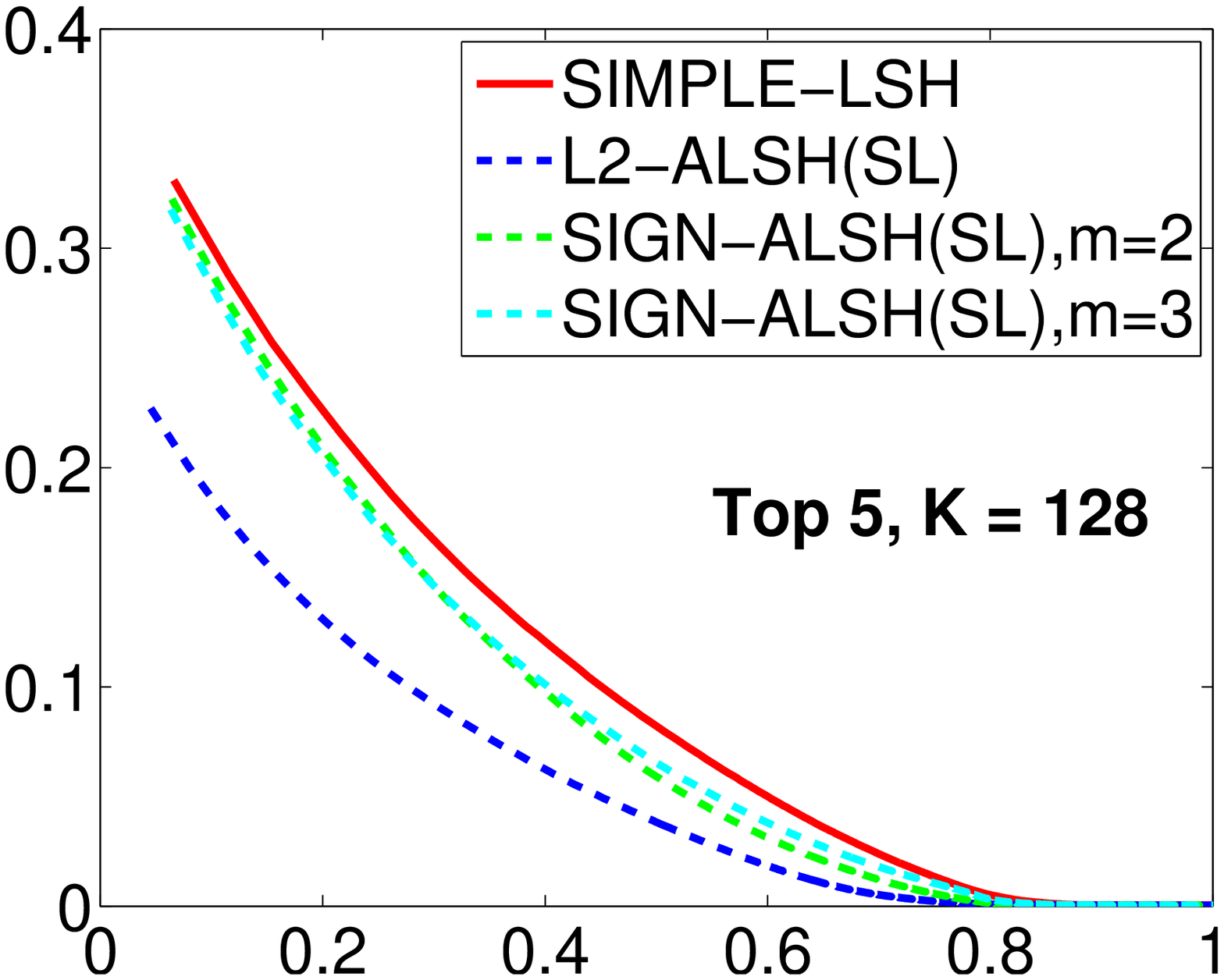}
\setlength{\epsfxsize}{0.235\textwidth}
\epsfbox{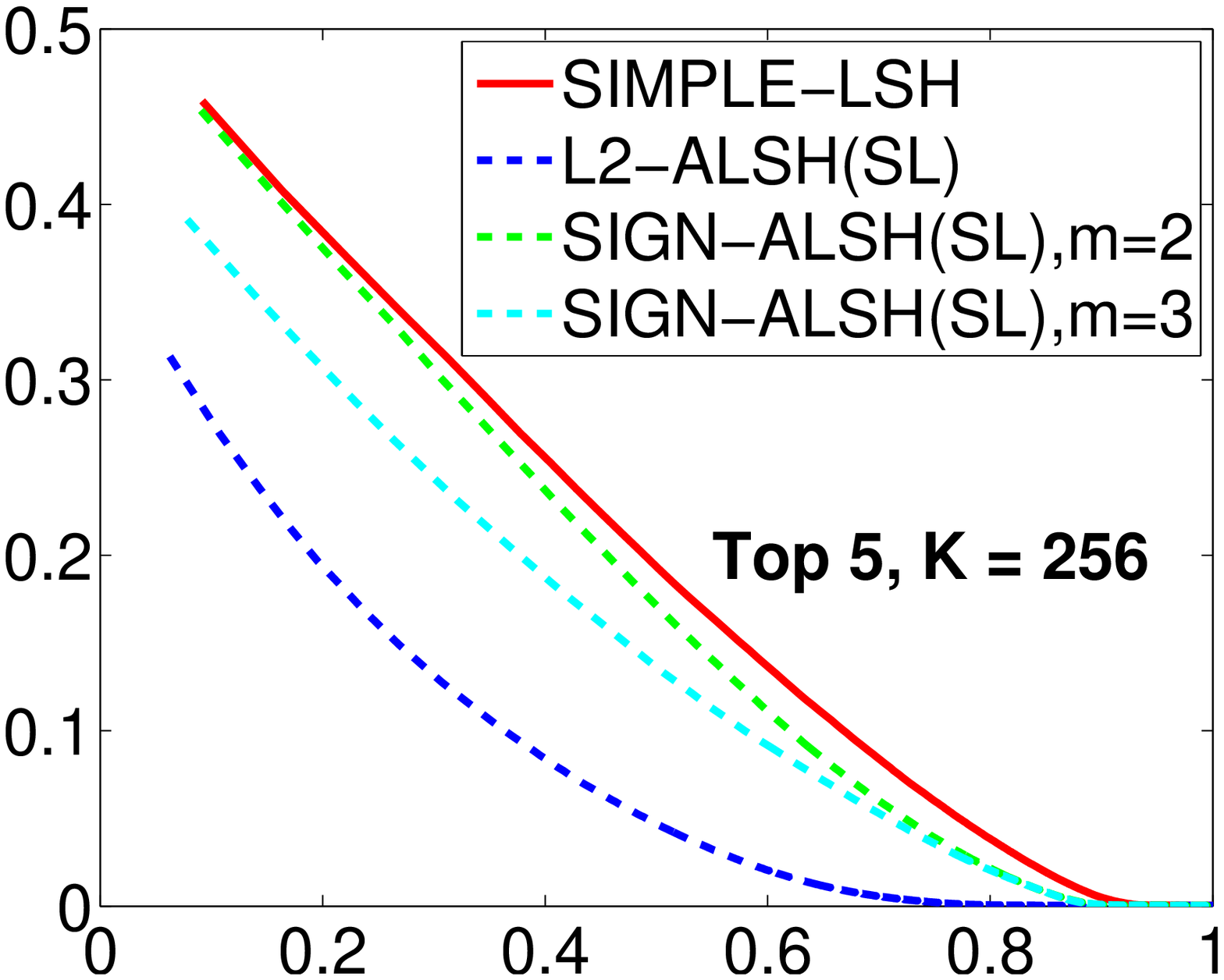}
\setlength{\epsfxsize}{0.235\textwidth}
\epsfbox{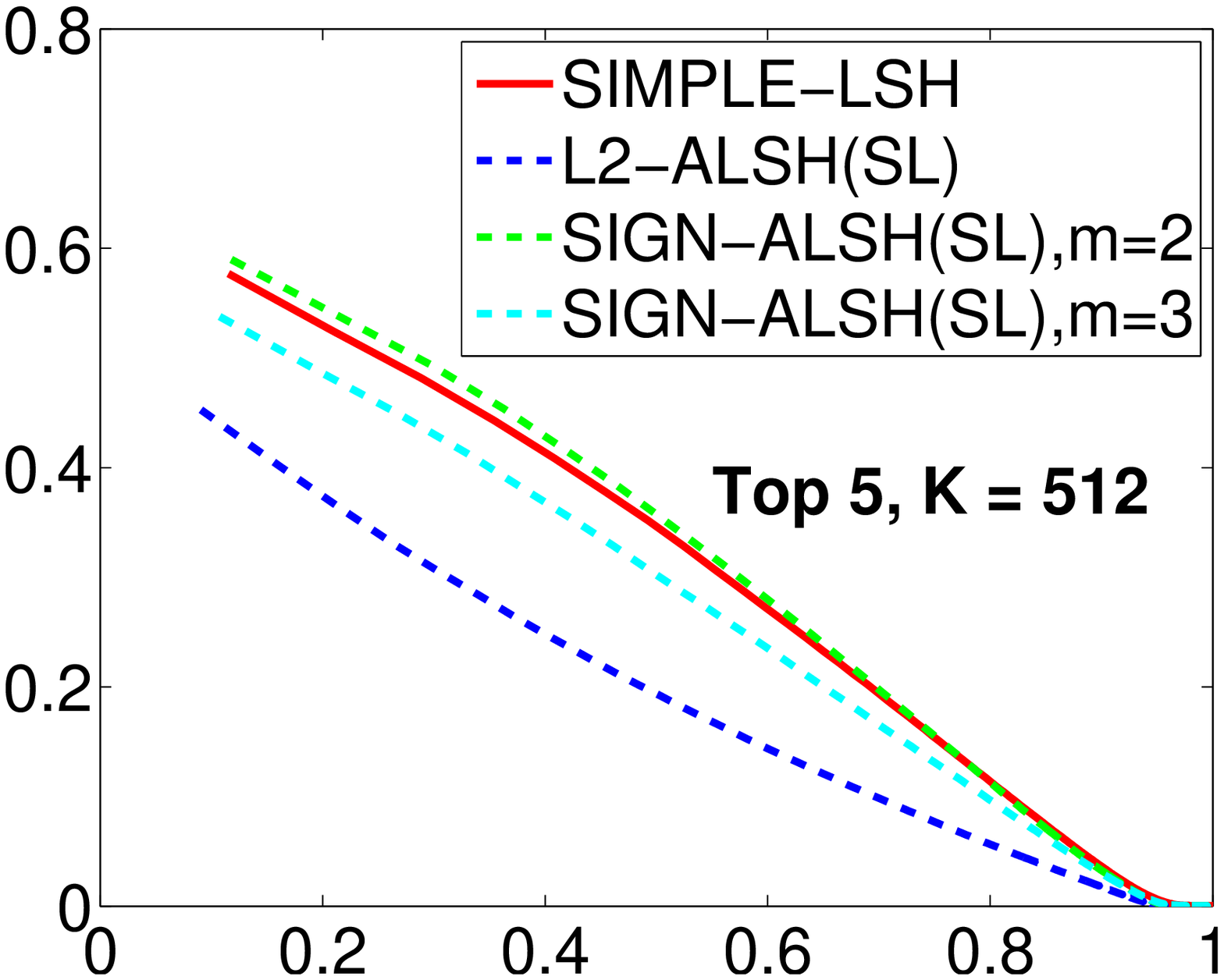}
}
\vspace{0.1in}
\hbox{ \centering \hspace{-0.2in}
\setlength{\epsfxsize}{0.25\textwidth}
\epsfbox{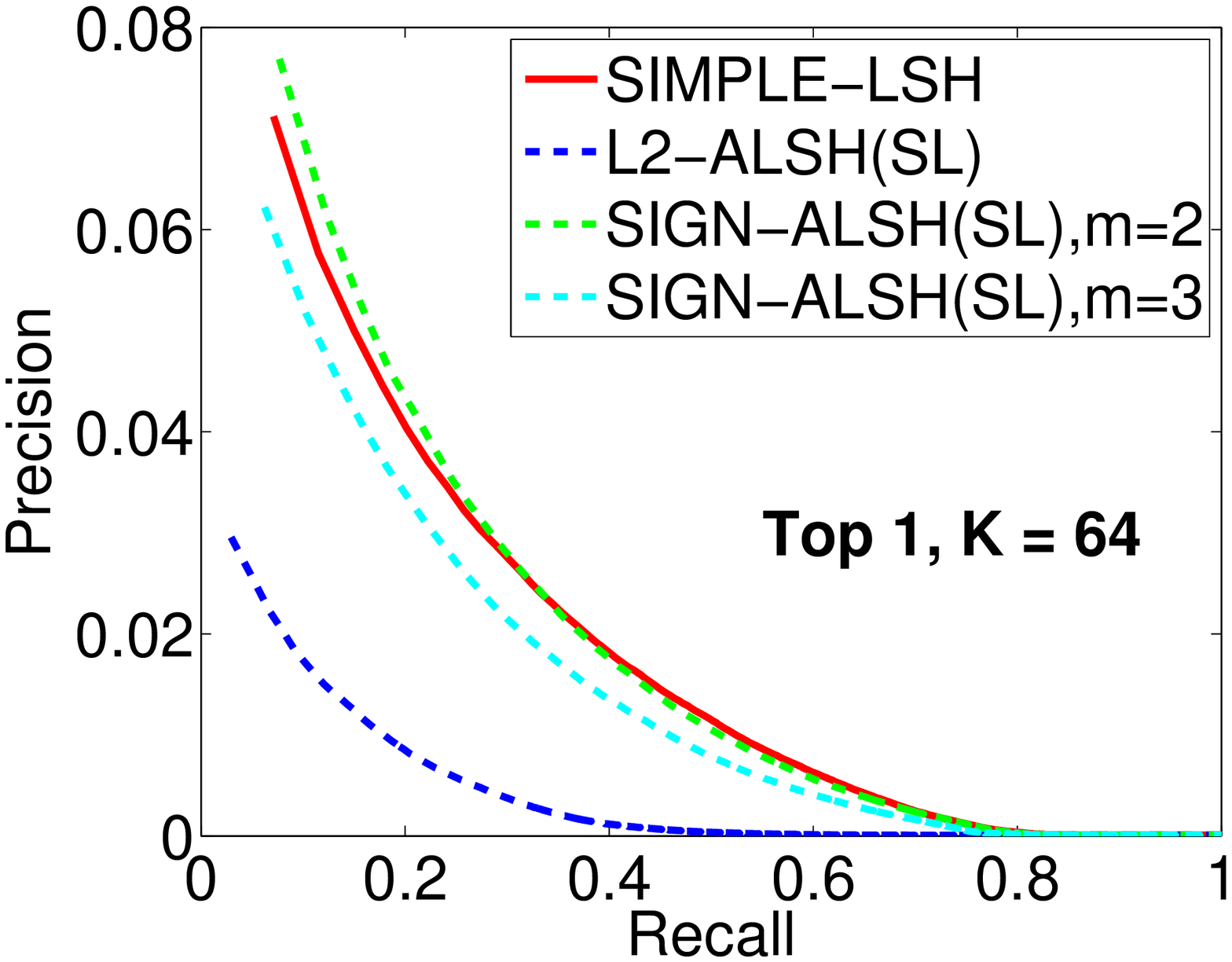}
\setlength{\epsfxsize}{0.235\textwidth}
\epsfbox{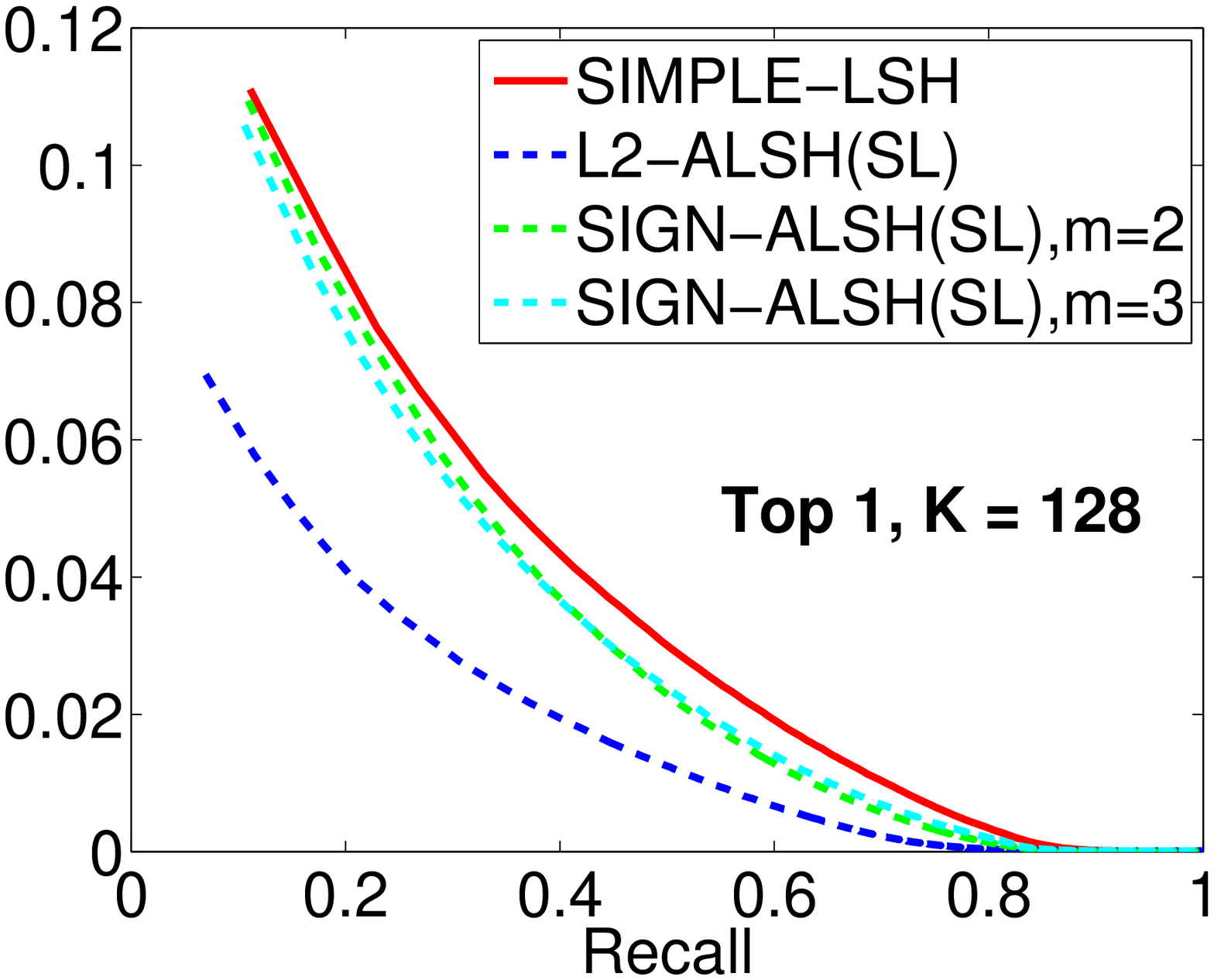}
\setlength{\epsfxsize}{0.235\textwidth}
\epsfbox{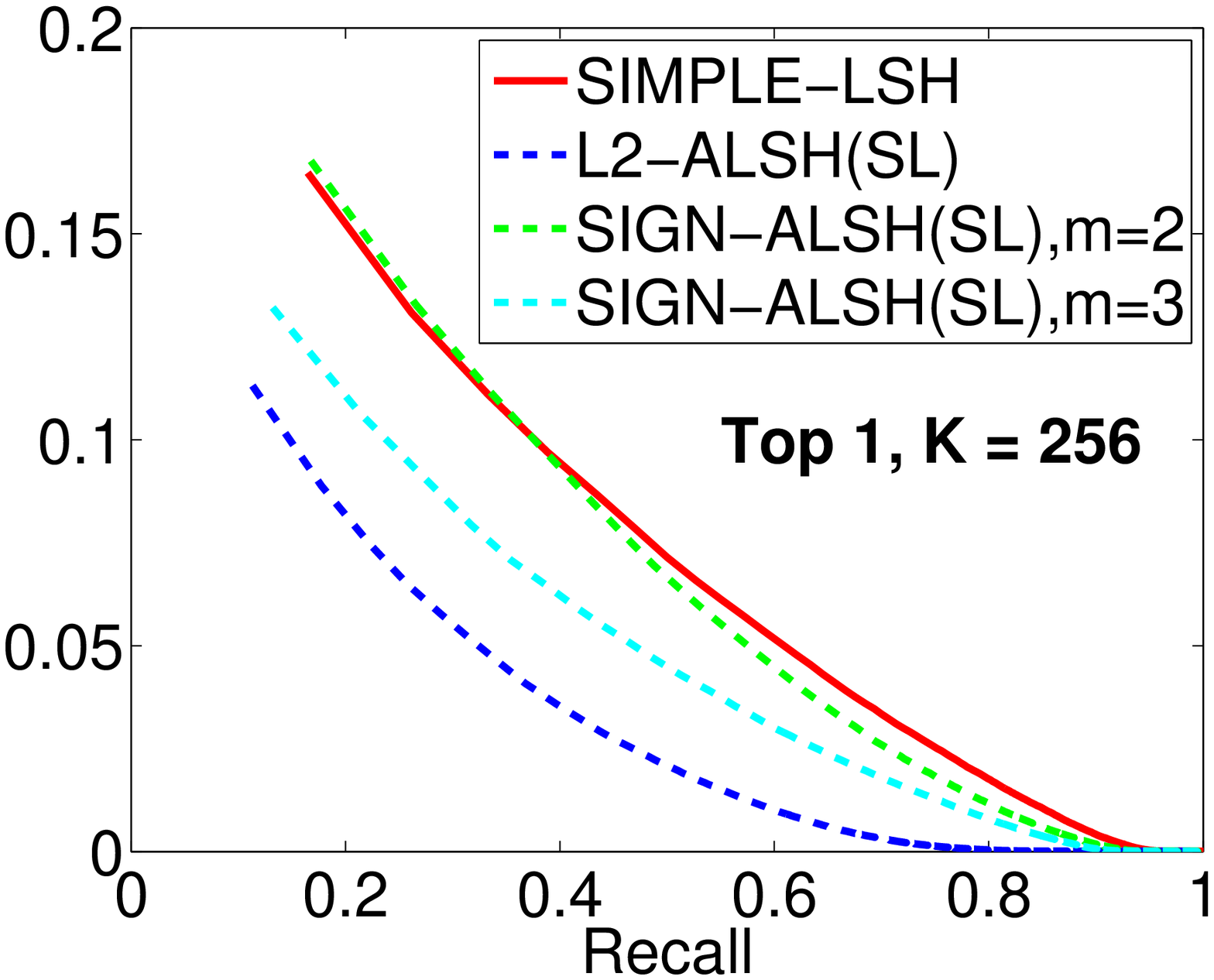}
\setlength{\epsfxsize}{0.235\textwidth}
\epsfbox{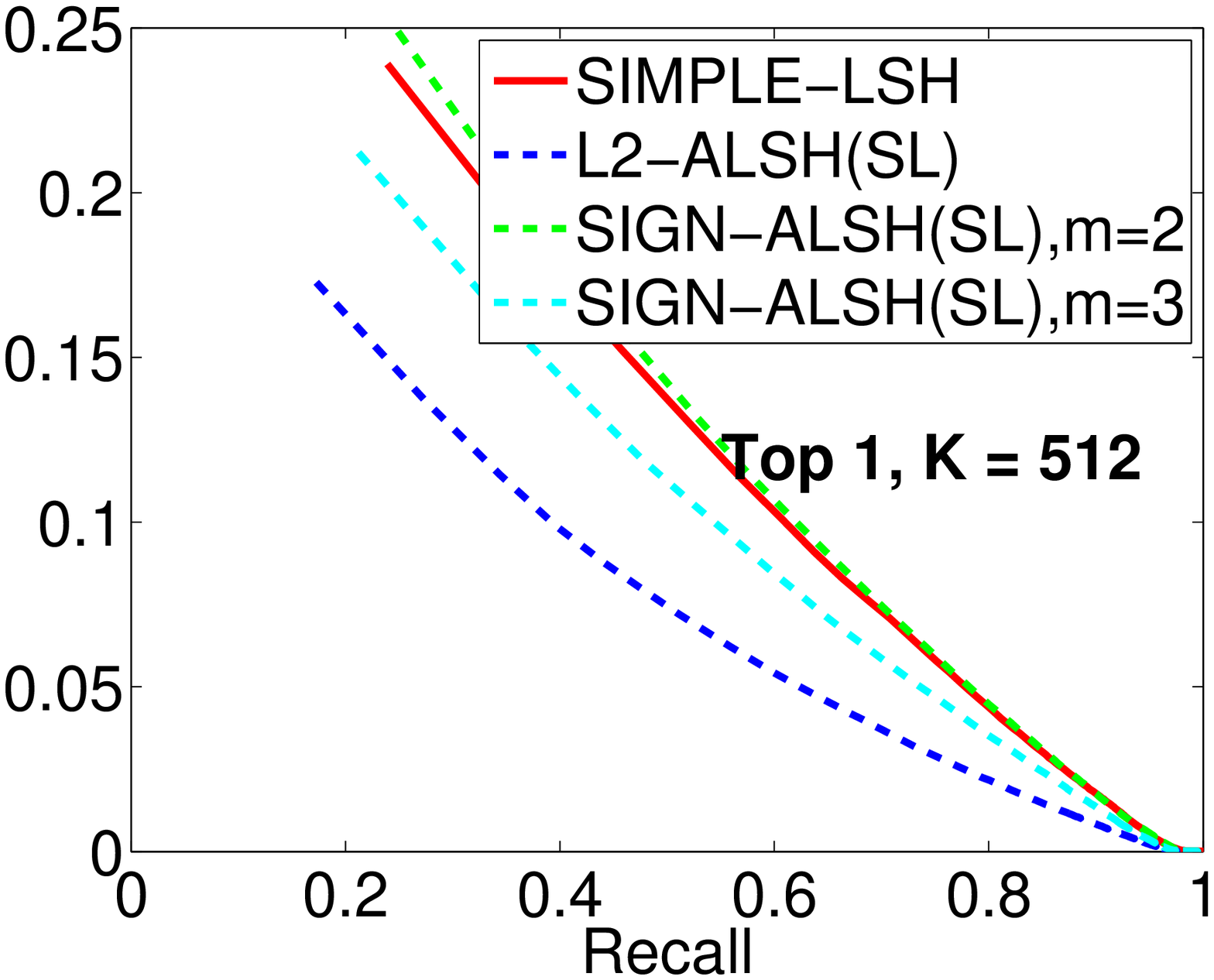}
}
\vspace{0.1in}
\caption{\small\textbf{Movielens}:  Precision-Recall curves (higher is better)
  of retrieving top $T$ items by hash code of length $K$. \textsc{simple-lsh}
  is parameter-free.  For \textsc{l2-alsh(sl)}, we fix the parameters
  $m=3$, $U=0.84$, $r=2.5$ and
  for \textsc{sign-alsh(sl)} we used two
different settings of the parameters: $m=2$, $U=0.75$ and $m=3$, $U=0.85$. \label{fig:movielens}
}
\end{figure*}
\subsection{Empirical Evaluation}
We also compared the hash functions empirically, following the exact
same protocol as \citet{shrivastava14}, using two collaborative
filtering datasets, Netflix and Movielens 10M.

For a given user-item matrix $Z$, we followed the pureSVD procedure
suggested by \citet{cremonesi10}: we first subtracted the overall
average rating from each individual rating and created the matrix $Z$
with these average-subtracted ratings for observed entries and zeros
for unobserved entries.  We then take a rank-$f$ approximation (top
$f$ singular components, $f=150$ for Movielens and $f=300$ for
Netflix) $Z\approx W\Sigma R^\top=Y$ and define $L=W\Sigma$ so that
$Y=LR^\top$.  We can think of each row of $L$ as the vector
presentation of a user and each row of $R$ as the presentation for an
item.

The database $S$ consists of all rows $R_j$ of $R$ (corresponding to
movies) and we use each row $L_i$ of $L$ (corresponding to users)
as a query.  That is, for each user $i$ we would like to find the top
$T$ movies, i.e.~the $T$ movies with highest $\inner{L_i,R_j}$, for
different values of $T$.

To do so, for each hash family, we generate hash codes of length $K$,
for varying lengths $K$, for all movies and a random selection of 60000
users (queries).  For each user, we sort movies in ascending order of
hamming distance between the user hash and movie hash, breaking up
ties randomly.  For each of several values of $T$ and $K$ we calculate
precision-recall curves for recalling the top $T$ movies, averaging
the precision-recall values over the 60000 randomly selected users.

In Figures \ref{fig:netflix} and \ref{fig:movielens}, we plot
precision-recall curves of retrieving top $T$ items by hash code of
length $K$ for Netflix and Movielens datasets where $T\in\{1,5,10\}$
and $K\in \{64,128,256,512\}$.  For \textsc{l2-alsh(sl)} we used
$m=3,U=0.83,r=2.5$, suggested by the authors and used in their
empirical evaluation.  For \textsc{sign-alsh(sl)} we used two
different settings of the parameters suggested by
\citet{shrivastava14b}: $m=2,U=0.75$ and $m=3,U=0.85$.
\textsc{simple-lsh} does not require any parameters.\removed{ For
  \textsc{$L_2$-lsh} we use $U=0.2,r=2$, which worked fairly well in
  all experiments.}

As can be seen in the Figures, \textsc{simple-lsh} shows a dramatic
empirical improvement over \textsc{l2-alsh(sl)}.  Following the
presentation of \textsc{simple-lsh} and the comparison with
\textsc{l2-alsh(sl)}, \citet{shrivastava14b} suggested the modified
hash \textsc{sign-alsh(sl)}, which is based on random projections, as
is \textsc{simple-lsh}, but with an asymmetric transform similar to
that in \textsc{l2-alsh(sl)}.  Perhaps not surprising,
\textsc{sign-alsh(sl)} does indeed perform almost the same as
\textsc{simple-lsh} (\textsc{simple-lsh} has only a slight advantage
on Movielens), however: (1) \textsc{simple-lsh} is simpler, and uses a
single symmetric lower-dimensional transformation $P(x)$; (2)
\textsc{simpler-lsh} is universal and parameter free, while
\textsc{sign-alsh(sl)} requires tuning two parameters (its authors
suggest two different parameter settings for use).  Therefor, we see
no reason to prefer \textsc{sign-alsh(sl)} over the simpler symmetric
option.

\section{Unnormalized Queries}\label{sec:sphere}

In the previous Section, we exploited asymmetry in the MIPS problem
formulation, and showed that with such asymmetry, there is no need for
the hash itself to be asymmetric.  In this Section, we consider LSH
for inner product similarity in a more symmetric setting, where we
assume no normalization and only boundedness.  That is, we ask whether
there is an LSH or ALSH for inner product similarity over 
$\XX_\bullet=\YY_\bullet=\left\{x\;\middle|\;\norm{x}\leq 1\right\}$.
Beyond a theoretical interest in the need for asymmetry in this fully
symmetric setting, the setting can also be useful if we are interested
in using sets $\XX$ and $\YY$ interchangeably as query and data sets.
In user-item setting for example, one might be also interested in
retrieving the top users interested in a given item without the need
to create a separate hash for this task.  

We first observe that there is no symmetric LSH for this setting.  We
therefore consider asymmetric hashes.  Unfortunately, we show that
neither $\textsc{l2-alsh(sl)}$ (nor $\textsc{sign-alsh(sl)}$) are ALSH
over $\XX_\bullet$.  Instead, we propose a parameter-free asymmetric
extension of \textsc{simple-lsh}, which we call \textsc{simple-alsh},
and show that it is a universal ALSH for inner product similarity over
$\XX_\bullet$.

To summarize the situation, if we consider the problem {\em
  asymmetrically}, as in the previous Section, there is no need for
the hash to be asymmetric, and we can use a single hash function.  But
if we insist on considering the problem {\em symmetrically}, we do
indeed have to use an {\em asymmetric} hash.
\subsection{No symmetric LSH}
We first show we do not have a symmetric LSH:
\begin{theorem}
  For any $0<S\leq1$ and $0<c<1$ there is no $(S,cS)$-LSH (by
  Definition \ref{def:lsh2}) for inner product similarity over
  $\XX_\bullet=\YY_\bullet=\left\{x\;\middle|\;\norm{x}\leq
    1\right\}$.
\end{theorem}
\begin{proof}
  The same argument as in \citet[][Theorem 1]{shrivastava14} applies:
  Assume for contradiction $h$ is an $(S,cS,p_1,p_2)$-LSH (with
  $p_1>p_2$).  Let $x$ be a vector such that $\norm{x}=cS<1$.  Let
  $q=x\in\XX_\bullet$ and $y=\frac{1}{c}x\in\XX_\bullet$.  Therefore,
  we have $q^\top x= cS$ and $q^\top y=S$. However, since $q=x$,
  $\PP_h(h(q)=h(x))=1\leq p_2 < p_1=\PP_h(h(q)=h(y))\leq 1$ and we get
  a contradiction.
\end{proof}
\subsection{L2-ALSH(SL)}
We might hope \textsc{l2-alsh(sl)} is a valid ALSH here.  Unfortunately,
whenever $S<(c+1)/2$, and so in particular for all $S<1/2$, it is not:
\begin{theorem}\label{thm:no-l2alsh}
  For any $0<c<1$ and any $0<S<(c+1)/2$, there are no $U,m$ and $r$
  such that \textsc{l2-alsh(sl)} is an $(S,cS)$-ALSH for inner
  product similarity over
  $\XX_\bullet=\YY_\bullet=\left\{x\;\middle|\;\norm{x}\leq
    1\right\}$.
\end{theorem}
\begin{proof}
  Let $q_1$ and $x_1$ be unit vectors such that
  $q_1^\top x_1=S$. Let $x_2$ be a unit vector and define $q_2=cSx_2$.
  \removed{We want to show that$\norm{P(x_2)-Q(q_2)}^2\leq
    \norm{P(x_1)-Q(q_1)}^2$.} For any $U$ and $m$:
\begin{align*}
\norm{P(x_2)-Q(q_2)}^2 &= \norm{q_2}+\frac{m}{4}+ \norm{Ux_2}^{2^{m+1}}-2q_2^\top x\\
&= c^2S^2+\frac{m}{4}+ U^{2^{m+1}}-2cSU\\
&\leq 1+\frac{m}{4}+ U^{2^{m+1}}-2SU\\
&=\norm{P(x_2)-Q(q_2)}^2
\end{align*}
where the inequality follows from $S<(c+1)/2$. Now, the same arguments as in Lemma \ref{lem:l2alshbound} using monotonicity of
collision probabilities in $\norm{P(x)-Q(q)}$ establish
\textsc{ls-alsh(sl)} is not an $(S,cS)$-ALSH.
\end{proof}
In Appendix~\ref{supp:sign}, we show a stronger negative result for
\textsc{sign-alsh(sl)}: for any $S>0$ and $0<c<1$, there are no $U,m$
such that \textsc{sign-alsh(sl)} is an $(S,cS)-ALSH$.

\subsection{SIMPLE-ALSH}
Fortunately, we can define a variant of $\textsc{simple-lsh}$, which
we refer to as \textsc{simple-alsh}, for this more general case where
queries are not normalized.  We use the pair of transformations:
\begin{eqnarray}
  \label{eq:simpleP}
  P(x) &=& \big[x;\sqrt{1-\|x\|_2^2}; 0\big]\\
  \nonumber
  Q(x) &=& \big[x;0;\sqrt{1-\|x\|_2^2}\big]
  \nonumber
\end{eqnarray}
and the random mappings $f(x)=h_a(P(x))$, $g(y)=h_a(Q(x))$, where
$h_a(z)$ is as in \eqref{eq:ha}. 
It is clear that by these definitions, we always have that for all $x,y\in\XX_\bullet$, $P(x)^\top Q(y)=x^\top y$ and $\norm{P(x)}=\norm{Q(y)}=1$.
\begin{theorem}\label{thm:smiple-alsh}
\textsc{simple-alsh} is a universal ALSH over $\XX_\bullet=\YY_\bullet=\left\{x\;\middle|\;\norm{x}\leq
    1\right\}$.  That is, for every $0<S,c<1$, it is an $(S,cS)$-ALSH
  over $\XX_\bullet,\YY_\bullet$.
\end{theorem}
\begin{proof}
  The choice of mappings ensures that for all $x,y\in\XX_\bullet$ we
  have $P(x)^\top Q(y)=x^\top y$ and $\norm{P(x)}=\norm{Q(y)}=1$, and
  so $\PP[h_a(P(x))=h_a(Q(y))] =1-\frac{\cos^{-1}(q^\top x)}{\pi}$.
  As in the proof of Theorem \ref{thm:smiple-lsh}, monotonicity of
  $1-\frac{\cos^{-1}(x)}{\pi}$ establishes the desired ALSH properties.
\end{proof}

~\citet{shrivastava14c} also showed how a modification of \textsc{simple-alsh} can be used for searching similarity measures such as set containment and weighted Jaccard similarity.

\section{Conclusion}

We provide a complete characterization of when symmetric and
asymmetric LSH are possible for inner product similarity:
\begin{itemize*}
\item Over $\RR^d$, no symmetric nor asymmetric LSH is possible.
\item For the MIPS setting, with normalized queries $\norm{q}=1$ and
  bounded database vectors $\norm{x}\leq 1$, a universal symmetric LSH
  is possible.
\item When queries and database vectors are bounded but not normalized,
  a symmetric LSH is not possible, but a universal asymmetric LSH is.
  Here we see the power of asymmetry.
\end{itemize*}
This corrects the view of \citet{shrivastava14}, who used the
nonexistence of a symmetric LSH over $\RR^d$ to motivate an asymmetric
LSH when queries are normalized and database vectors are bounded, even
though we now see that in these two settings there is actually no
advantage to asymmetry.  In the third setting, where an asymmetric
hash is indeed needed, the hashes suggested by
\citet{shrivastava14,shrivastava14b} are not ALSH, and a
different asymmetric hash is required (which we provide).
Furthermore, even in the MIPS setting when queries are normalized (the
second setting), the asymmetric hashes suggested by
\citet{shrivastava14,shrivastava14b} are {\em not} universal
and require tuning parameters specific to $S,c$, in contrast to
\textsc{simple-lsh} which is symmetric, parameter-free and universal.

It is important to emphasize that even though in the MIPS setting an
asymmetric hash, as we define here, is not needed, an asymmetric view
of the problem {\em is} required.  In particular, to use a symmetric
hash, one {\em must} normalize the queries but not the database
vectors, which can legitimately be viewed as an asymmetric operation
which is part of the hash (though then the hash would not be, strictly
speaking, an ALSH).  In this regard \citet{shrivastava14} do
indeed successfully identify the need for an asymmetric view of MIPS,
and provide the first practical ALSH for the problem.

\subsubsection*{Acknowledgments}
This research was partially funded by NSF award IIS-1302662.

\bibliographystyle{apa}
\bibliography{ref}
\appendix

\section{Another variant}\label{supp:sign}

To benefit from the empirical advantages of random projection hashing,
\citet{shrivastava14b} also proposed a modified asymmetric LSH, which
we refer to here as \textsc{sign-alsh(sl)}.  \textsc{sign-alsh(sl)}
uses two different mappings $P(x)$, $Q(q)$, similar to those of
\textsc{l2-alsh(sl)}, but then uses a random projection hash $h_a(x)$,
as is the one used by \textsc{simple-lsh}, instead of the quantized
hash used in \textsc{l2-alsh(sl)}.  In this appendix we show that
our theoretical observations about \textsc{l2-alsh(sl)}
are also valid for \textsc{sign-alsh(sl)}.

\textsc{sign-alsh(sl)} uses the pair of mappings:
\begin{equation}
  \label{eq:theirPQ2}
  \begin{aligned}
    P(x)&=[Ux;1/2-\norm{Ux}^2;\dots;1/2-\norm{Ux}^{2^m}]\\
    Q(y)&=[y;0;0;\dots;0],
  \end{aligned}
\end{equation}
where $m$ and $U$ are parameters, as in \textsc{l2-alsh(sl)}.
\textsc{sign-alsh(ls)} is then given by
$f(x)=h_a(P(x))$, $g(y)=h_a(Q(x))$, where $h_a$ is the random projection
hash given in \eqref{eq:ha}.  \textsc{sign-alsh(ls)} therefor depends
on two parameters, and uses a binary alphabet $\Gamma=\{\pm 1\}$.

In this section, we show that, like \textsc{l2-alsh(ls)},
\textsc{sign-alsh(ls)} is not a universal ALSH over
$\XX_\bullet,\YY_\circ$, and moreover for any $S>0$ and $0<c<1$ it is not an $(S,cS)$-ALSH over
$\XX_\bullet=\YY_\bullet$:
\begin{lemma}
For any $m,U,r,$ and for any $0<S<1$ and
$$
\min\bigg\{\sqrt{1-\frac{U^{2^{m+1}}(1-S^{2^{m+1}})}{U^{2^{m+1}}+m/4}},\frac{\sqrt[2^{m+1}]{\frac{(m/2)}{2^{m+1}-2}}}{SU}\bigg\} \leq c < 1
$$
\textsc{sign-alsh(sl)} is not an $(S, cS)$-ALSH for inner product similarity over $\XX_\bullet=\left\{ x \middle| \norm{x}\leq 1 \right\}$ and
  $\YY_\circ=\left\{ q \middle| \norm{q}=1 \right\}$.
\end{lemma}
\begin{proof}
  Assume for contradiction that:
  $$
  \sqrt{1-\frac{U^{2^{m+1}}(1-S^{2^{m+1}})}{U^{2^{m+1}}+m/4}} \leq c < 1
  $$
   and \textsc{sign-alsh(sl)} is an $(S,cS)$-ALSH.  For any query
  point $q\in \YY_\circ$, let $x\in \XX_\bullet$ be a vector s.t.~$q^\top x=S$ and
  $\|x\|_2=1$ and let $y=cSq$, so that $q^\top y = cS$. We have that:
\begin{align*}
\frac{( P(y)^\top Q(q) )^2}{\norm{P(y)}^2} &= \frac{c^2S^2U^2}{m/4+\norm{y}^{2^{m+1}}}\\
&= \frac{c^2S^2U^2}{m/4+(cSU)^{2^{m+1}}}\\
\intertext{Using $1-\frac{U^{2^{m+1}}(1-S^{2^{m+1}})}{U^{2^{m+1}}+m/4}\leq c^2 < 1$:}
&>  \frac{c^2S^2U^2}{m/4+(SU)^{2^{m+1}}}\\
&\geq \frac{S^2U^2}{m/4+U^{2^{m+1}}}\\
&=\frac{( P(x)^\top Q(q) )^2}{\norm{P(x)}^2}
\end{align*}
 The monotonicity of $1-\frac{\cos^{-1}(x)}{\pi}$ establishes a contradiction. To get the other bound on $c$, let $\alpha_m = \sqrt[2^{m+1}]{\frac{(m/2)}{2^{m+1}-2}}$ and assume for contradiction that:
  $$
  \frac{\alpha_m}{SU}=\frac{\sqrt[2^{m+1}]{\frac{(m/2)}{2^{m+1}-2}}}{SU}\leq c < 1
  $$
  and \textsc{sign-alsh(sl)} is an $(S,cS)$-ALSH.  For any query
  point $q\in \YY_\circ$, let $x\in \XX_\bullet$ be a vector s.t.~$q^\top x=S$ and
  $\|x\|_2=1$ and let $y=(\alpha_m/U)q$.
By the monotonicity of $1-\frac{\cos^{-1}(x)}{\pi}$, to get a contradiction is enough to show that 
$$
\frac{P(x)^\top Q(q)}{\norm{P(x)}}\leq \frac{P(y)^\top Q(q)}{\norm{P(y)}}
$$
We have:
\begin{align*}
\frac{( P(y)^\top Q(q) )^2}{\norm{P(y)}^2} &= \frac{\alpha_{m}^2}{m/4+\norm{\alpha_{m}}^{2^{m+1}}}\\
&= \frac{\sqrt[2^{m}]{\frac{(m/2)}{2^{m+1}-2}}}{m/4+\frac{(m/2)}{2^{m+1}-2}}\\
\intertext{Since this is the maximum value of the function $f(U)=U^2/(m/4+U^{2^{m+1}})$:}
&\geq \frac{U^2}{m/4+U^{2^{m+1}}}\\
&\geq \frac{S^2U^2}{m/4+U^{2^{m+1}}}\\
&=\frac{( P(x)^\top Q(q) )^2}{\norm{P(x)}^2}
\end{align*}
which is a contradiction.
\end{proof}

\begin{cor}
  For any $U,m$ and $r$, \textsc{sign-alsh(sl)} is not a universal ALSH
  for inner product similarity over $\XX_\bullet=\left\{ x \middle|
    \norm{x}\leq 1 \right\}$ and $\YY_\circ=\left\{ q \middle| \norm{q}=1
  \right\}$.  Furthermore, for any $c<1$, and any choice of $U,m,r$
  there exists $0<S<1$ for which \textsc{sign-alsh(sl)} is not an
  $(S,cS)$-ALSH over $\XX_\bullet,\YY_\circ$, and for any $S<1$ and any choice of
  $U,m,r$ there exists $0<c<1$ for which \textsc{sign-alsh(sl)} is not
  an $(S,cS)$-ALSH over $\XX_\bullet,\YY_\circ$.
\end{cor}

\begin{lemma}
For any $S>0$ and $0<c<1$ there are no $U$ and $m$ such that \textsc{sign-alsh(sl)} is an $(S,cS)$-ALSH for inner product similarity over $\XX_\bullet=\YY_\bullet=\left\{x\;\middle|\;\norm{x}\leq
    1\right\}$.
\end{lemma}
\begin{proof}
Similar to the proof of Theorem \ref{thm:no-l2alsh}, for any $S>0$ and $0<c<1$, let $q_1$ and $x_1$ be unit vectors such that
  $q_1^\top x_1=S$. Let $x_2$ be a unit vector and define $q_2=cSx_2$. For any $U$ and $m$:
\begin{align*}
\frac{P(x_2)^\top Q(q_2) }{\norm{P(x_2)}\norm{Q(q_2)}} &= \frac{cSU}{cS\sqrt{ m/4+\norm{U}^{2^{m+1}}}}\\
&= \frac{U}{\sqrt{ m/4+\norm{U}^{2^{m+1}}}}\\
&\geq\frac{SU}{\sqrt{ m/4+\norm{U}^{2^{m+1}}}}\\
&=\frac{P(x_1)^\top Q(q_1) }{\norm{P(x_1)}\norm{Q(q_1)}} 
\end{align*}  
Now, the same arguments as in Lemma \ref{lem:l2alshbound} using monotonicity of
collision probabilities in $\norm{P(x)-Q(q)}$ establish
\textsc{sign-alsh(sl)} is not an $(S,cS)$-ALSH.
\end{proof}

\section{Max-norm and margin complexity}\label{supp:max}
\subsubsection*{Max-norm}
The max-norm (aka $\gamma_2\!\!:\!\!\ell_1\!\!\rightarrow \!\!\ell_\infty$ norm) is defined as \citep{srebro05b}:
\begin{equation}\label{eq:maxnorm}
\maxnorm{X} = \min_{X=UV^\top} \max(\|U\|_{2,\infty}^2, \|V\|_{2,\infty}^2)
\end{equation}
where $\|U\|_{2,\infty}$ is the maximum over $\ell_2$ norms of rows of matrix $U$, i.e. $\|U\|_{2,\infty}= \max_i \norm{U[i]}$. 

For any pair of sets $(\{x_i\}_{1\leq i \leq n},\{y_i\}_{1\leq i \leq
  m})$ and hashes $(f,g)$ over them, let P be the collision
probability matrix, i.e. $P(i,j)=\PP[f(x_i)=g(y_j)]$. In the following
lemma we prove that $\maxnorm{P} \leq 1$:
\begin{lemma}
For any two sets of objects and hashes over them, if $P$ is the collision probability matrix, then $\maxnorm{P} \leq 1$.
\end{lemma}
\begin{proof}
For each $f$ and $g$, define the following biclustering matrix:
\begin{equation}
\kappa_{f,g}(i,j)=
\begin{cases}
1 & f(x_i)=g(y_j)\\
0 & \text{otherwise.}
\end{cases}
\end{equation}
For any function $f:\ZZ \rightarrow \Gamma$, let $R_f\in \{0,1\}^{n \times |\Gamma|}$ be the indicator of the values of function $f$:
\begin{equation}
R_{h}(i,\gamma)=
\begin{cases}
1 & h(x_i)=\gamma\\
0 & \text{otherwise,}
\end{cases}
\end{equation}
and define $R_g \in \{0,1\}^{m\times |\Gamma|}$ similarly. 
It is easy to show that $\kappa_{f,g} = R_f R_g^\top$ and since
$\|R_f\|_{2,\infty}=\|R_g\|_{2,\infty}=1$, by the definition of the
max-norm, we can conclude that $\maxnorm{\kappa_{f,g}}\leq 1$.  But
the collision probabilities are given by $P=\EE[\kappa_{f,g}]$, and so
by convexity of the max-norm and Jensen's inequality,
$\maxnorm{P}=\maxnorm{\EE[\kappa_{f,g}]}\leq
\EE[\maxnorm{\kappa_{f,g}}]\leq 1$.
\end{proof}
It is also easy to see that $1_{n\times n}=RR^\top$ where $R=1_{n\times 1}$. Therefore for any $\theta\in \RR$,
$$
\maxnorm{ \theta_{n \times n} } = \theta \maxnorm{ 1_{n \times n} } \leq |\theta|
$$
\subsubsection*{Margin complexity}

For any sign matrix $Z$, the margin complexity of $Z$ is defined as:
\begin{align}
\min_{Y} &\;\;\;\maxnorm{Y}\\
\nonumber
\text{s.t.} &\;\;\;Y(i,j)X(i,j) \geq 1\;\;\; \forall i,j
\end{align}

Let $Z\in \{\pm 1 \}^{N\times N}$ be a sign matrix with +1 on and above the diagonal and -1 below it. \citet{forster03} prove that the margin complexity of matrix $Z$ is $\Omega(\log N)$.

\end{document}